\pdfoutput=1
\documentclass[11pt]{article}
\usepackage{amsfonts}
\usepackage{mathrsfs}
\usepackage{amssymb,amsmath,amsthm}
\usepackage{graphicx}
\usepackage[dvips, paper=letterpaper, top=0.95in, bottom=.6in, left=0.9in, right=0.95in, nohead, includefoot, footskip=.4in]{geometry}
\usepackage[authoryear,round]{natbib}
\usepackage{environ}
\usepackage{color}
\usepackage{epsfig}
\usepackage{caption}
\usepackage{verbatim}
\usepackage{subcaption}
\usepackage{float}
\usepackage{setspace}
\usepackage{enumitem}
\setlist[enumerate]{itemsep=0mm}
\usepackage{calc}
\usepackage[pdfpagemode=UseOutlines,hyperfootnotes=false, plainpages=false]{hyperref}
\usepackage{mathtools}
\usepackage{dsfont}
\usepackage[english]{babel}
\usepackage{amsthm}
\usepackage{bm}
\usepackage{breqn}
\usepackage{bigints}
\usepackage{tikz}
\usepackage{nicefrac}
\usepackage{xurl}
\usepackage{algorithmic}
\usepackage[ruled,vlined]{algorithm2e}

\usepackage[font=normalsize,labelfont=bf]{caption}
\usepackage[font=footnotesize,labelfont=bf]{subcaption}
\usepackage{sectsty}
\allsectionsfont{\normalfont\sffamily\bfseries}
\sffamily

\newtheoremstyle{theoremsansserif} 
    {\topsep}                    
    {\topsep}                    
    {\itshape}                   
    {}                           
    {\sffamily\bfseries }        
    {.}                          
    {.5em}                       
    {}  

\theoremstyle{theoremsansserif}

\newtheorem{lemma}{Lemma}

\newtheorem{proposition}{Proposition}

\newtheorem{definition}{Definition}

\newtheorem*{example*}{Example}
\newtheorem{assumption}{Assumption}
\newtheorem{theorem}{Theorem}

\newcommand{\R}{\mathbb{R}}
\newcommand{\D}{\mathcal{D}}
\newcommand{\PP}{\mathcal{P}}
\renewcommand{\S}{\mathcal{S}}

\newcommand{\X}{\mathcal{X}}

\newcommand{\E}{\mathbb{E}}
\newcommand{\tA}{\zeta_A}

\newcommand{\W}{\mathcal{W}}
\newcommand{\dbtilde}[1]{\tilde{\raisebox{0pt}[0.85\height]{$\tilde{#1}$}}}

\DeclareMathOperator{\argmin}{argmin}
\DeclareMathOperator{\argmax}{argmax}
\DeclareMathOperator{\regret}{\operatorname{Regret}}
\DeclareMathOperator{\opt}{OPT}
\DeclareMathOperator{\fluid}{FLUID}
\DeclareMathOperator{\coeff}{coeff}
\newcommand{\noneurips}[1]{#1}

\begin{document}
\title{\sf\textbf{Robust Budget Pacing with a Single Sample}}
\author{
\sf Santiago R. Balseiro\\
\sf Columbia University \\
\small\texttt{}
\and
\sf Rachitesh Kumar\footnote{Correspondence to \textcolor{blue}{\texttt{rachitesh@cmu.edu}}. Work done as a Student Researcher at Google Research.}\\
\sf Carnegie Mellon University \\
\small\texttt{}
\and
\sf Vahab Mirrokni \\
\sf Google Research\\
\small\texttt{}
\and
\sf Balasubramanian Sivan \\
\sf Google Research\\
\small\texttt{}
\and
\sf Di Wang \\
\sf Google Research\\
\small\texttt{}}

\date{\vspace{-2em}}

\maketitle

\allowdisplaybreaks

\thispagestyle{empty}

\setcounter{page}{1}

\setstretch{1.5}

\begin{abstract}
    Major Internet advertising platforms offer budget pacing tools as a standard service for advertisers to manage their ad campaigns. Given the inherent non-stationarity in an advertiser's value and also competing advertisers' values over time, a commonly used approach is to learn a target expenditure plan that specifies a target spend as a function of time, and then run a controller that tracks this plan. This raises the question: \emph{how many historical samples are required to learn a good expenditure plan}? We study this question by considering an advertiser repeatedly participating in $T$ second-price auctions, where the tuple of her value and the highest competing bid is drawn from an unknown time-varying distribution. The advertiser seeks to maximize her total utility subject to her budget constraint. Prior work has shown the sufficiency of \emph{$T\log T$ samples per distribution} to achieve the optimal $O(\sqrt{T})$-regret. We dramatically improve this state-of-the-art and show that \emph{just one sample per distribution} is enough to achieve the near-optimal $\tilde O(\sqrt{T})$-regret, while still being robust to noise in the sampling distributions.
\end{abstract}
\section{Introduction}

Online advertising is the economic engine of the internet. It allows platforms like Google, Facebook, and TikTok to fund services that are free at the point of delivery while providing businesses the ability to target their ads to relevant users. When a user logs onto these platforms, an auction is run among the interested advertisers to determine which ad is displayed to her. Given the scale of the internet, advertisers are typically interested in thousands (if not millions) of auctions every day. If they participate in all such auctions, they will likely spend far beyond their budget. This necessitates the need for budget management, which is our focus.

\textbf{Target spend plans to address non-stationarities.} Given the inherent non-stationarities that exist over time, in the volume of traffic, in the demographic of users that visit, the rates of conversion etc., an advertiser's own value and that of the competing advertisers' values are also non-stationary. To deal with this non-stationarity, budget management systems compute a target expenditure plan~\citep{facebookguide, kumar2022optimal}. The latter is a function of time that specifies the recommended amount of spend at each point of time. I.e., such a plan distributes the cumulative daily/weekly budget into smaller chunks of time, appropriately capturing the non-stationarity. A pacing algorithm like a controller is used to track the plan. This raises the question: how much data is required to come up with a good target spend plan? 

\textbf{Our model.} We study this question by considering a budget-constrained advertiser participating in $T$ second-price auctions. The advertiser seeks to maximize her utility while respecting her budget (formal model in Section~\ref{sec:model}). We model the tuple, of an advertiser's value and that of her highest competing bid, at time $t$, as being generated from unknown independent time-varying distributions $\{\PP_t\}_{t=1}^T$. This is a major departure from the stochastic budget-management/resource-allocation literature, which for the most part assumes stationary distributions---an assumption that rarely holds in practice as daily traffic patterns are non-stationary~\citep{zhou2019robust}. If nothing is known about the distributions $\{\PP_t\}_t$ and they can be chosen adversarially,~\citet{balseiro2019learning} showed that every algorithm must incur linear regret against the hindsight-optimal benchmark. Their impossibility result also shows that budget management does not fall under the purview of online convex optimization (OCO) and is vastly different, given that one can achieve $O(\sqrt{T})$ regret for adversarially-chosen input in OCO. Thankfully, additional information is usually available in the form of historical samples. For example, both intraday and interday internet activity (and consequently the user traffic in which an advertiser is interested) tends to be similar week-on-week. However, this periodicity is never exact, i.e., it is too stringent to assume these historical samples are drawn from exactly the same distributions as $\{\PP_t\}_t$. To reflect this reality, in this work, we assume that we have samples from distributions $\{\tilde \PP_t\}$, which are potentially different from $\{\PP_t\}$, and develop algorithms that are robust to differences between them. We call an algorithm robust if its regret degrades smoothly as a function of the total Wasserstein distance between the two sequences of distributions. Crucially, classical training-based algorithms that learn a dual solution from samples and use it for bidding (see Algorithm~\ref{alg:learned-dual}, inspired by \citealt{devanur2009adwords,agrawal2014dynamic}) are not robust (see Section~\ref{sec:dual}). 

\textbf{Our result.} Our main contribution is to show that Dual Follow-The-Regularized-Leader (FTRL) is robust and achieves the near optimal $\tilde O(\sqrt{T})$ regret even when it has access to \emph{just one sample from each distribution} $\tilde \PP_t$, dramatically improving over the prior $T\log T$ samples from each distribution~\citep{jiang2020online} to achieve $O(\sqrt{T})$ regret. 

\textbf{Key insights.} Our algorithm uses the samples to estimate the ideal amount of expenditure to target in each auction and then uses Dual FTRL to follow these targets by shading the advertiser's values appropriately. The key insight driving the reduction in sample complexity from $T\log T$ to $1$ per distribution is the following. Prior work by~\citet{jiang2020online} first learns the sampling distributions $\{\tilde \PP_t\}_t$ from the samples, then computes the optimal duals on the learned distributions, and finally uses those duals to compute the target expenditures. Our insight is that it is not necessary to learn the entire sampling distribution. Instead, it is far more efficient to directly learn the duals from the samples and construct target expenditure based on those duals (i.e., setting target expenditure at $t$ to be what the dual-based solution consumes at $t$). Beyond being very efficient with samples, learning the duals from the samples also guarantees robustness to shifts between $\{\tilde \PP_t\}_t$ and $\{\PP_t\}_t$.



\textbf{Practical implications.} Consider $T$ representing a week's worth of auctions. While a total of $T\log T$ samples may be possible to obtain by looking at the past few weeks (given that each week has about $T$ samples to offer), requiring $T\log T$ samples \emph{per distribution} calls for looking at many months into the past. This is because getting a sample for a distribution, where a distribution could correspond to say a particular hour (e.g. Monday 10 AM), entails looking at that same hour from the past week. Asking for $T \log T$ samples for any given hour, when $T$ samples is what we get for the entire week, clearly requires looking at the numerous months into the past. Apart from posing huge storage and operational challenges, given that traffic pattern shifts over time, even gradual shifts would significantly degrade quality as one moves too much into the past. Our ask of one sample per distribution requires looking at just the past week and getting one sample for each hour.


\textbf{Near optimality: $O(\sqrt{T})$ vs $\tilde{O}(\sqrt{T})$.} Furthermore, our regret guarantee is near-optimal in light of the $\Omega(\sqrt{T})$ lower bound established in \citet{arlotto2019uniformly}, which holds even in the much easier case when the distributions $\{\PP_t\}$ are identical (i.e., i.i.d. setting) and known ahead of time. Why is our $\tilde{O}(\sqrt{T})$ regret separated from the $O(\sqrt{T})$ regret in the lower bound (and the matching upper bound in~\citet{jiang2020online})? Indeed, we can also obtain $O(\sqrt{T})$ regret when given $\log T$ samples per distribution, or in the other direction, the~\citet{jiang2020online} paper also gets only $\tilde{O}(\sqrt{T})$ regret when given only $T$ samples per distribution rather than $T\log T$ samples. In other words, we get a factor $T$ reduction in samples: either from $T$ to $1$ sample while maintaining the same regret of $\tilde{O}(\sqrt{T})$ or from $T\log T$ to $\log T$ samples while maintaining the same regret of $O(\sqrt{T})$.

\textbf{Primal vs dual approach.} The recent works of~\citet{banerjee2020constant} and \citet{banerjee2020uniform} show how to obtain a constant regret with a single historical trace by repeatedly solving the primal. We note that they crucially rely on the total number of types being small. When the number of types is very large as in Internet advertising where the type models user features, intent, time of day, location etc., a primal based algorithm is not even well defined because one would not have even seen all the types in the samples. Thus, a primal approach is out of question for us. More on this in Appendix~\ref{sec:related-work}.

\textbf{Technical contributions.} We achieve our result by developing a novel dual-iterate coupling lemma (see Lemma~\ref{lemma:coupling}) and leveraging it to analyze a leave-one-out thought experiment designed to break challenging correlations which arise from working with one sample per distribution (see Subsection~\ref{subsec:R3} for details). Additionally, we also prove a novel regret decomposition for Dual FTRL (Theorem~\ref{thm:dual-descent-regret}), which may be of independent interest. Finally, our algorithm does not require solving large linear programs and can be implemented efficiently (see Subsection~\ref{subsec:putting-together}), which is critical for online advertising since each auction runs in a few milliseconds. Due to space constraints, we directly move onto a formal description of our model next, and refer the reader to Appendix~\ref{sec:related-work} for a discussion on other related work.
\section{Model}
\label{sec:model}

\textbf{Notation.} We use $\R_+$ and $\R_{++}$ to denote the set of non-negative real numbers and the set of positive real numbers respectively. For $n \in \mathbb{N}$, we use $[n] = \{1, \dots, n\}$ to denote the set of positive integers less than or equal to $n$. We use $\W(\cdot, \cdot)$ to denote the Wasserstein distance between two distributions under the metric with which the sample space is endowed. 

\textbf{Online Allocation with a Single Resource and Budget Management.} For ease of exposition, we will prove our results for the more general single-resource online allocation problem with linear rewards/consumptions. It is well-known that bidding in repeated second-price auctions with budgets can be modelled as an instance of this online allocation problem (e.g., see \citealt{balseiro2022best}, or our Appendix~\ref{sec:application}). It also captures the stochastic multi-secretary problem \citep{arlotto2019uniformly} as a special case.

Consider a decision maker with an initial budget $B \in \R_{++}$ of a resource, whose goal is to optimally spend it on $T$ sequentially arriving requests. Each request $\gamma = (f,b)$ is comprised of a linear reward function $f: \X \to \R_+$ such that $f(x) = \coeff(f) \cdot x$, and a linear resource consumption function $b: \X \to \R_+$ such that $b(x) = \coeff(b) \cdot x$; where $\X \subseteq \R_+$ is a compact set which denotes the space of possible actions of the decision maker. We will use $\S$ to denote the set of all possible requests and $\Delta(S)$ to denote the set of distributions over $\S$. Moreover, we endow $\S$ with the following metric $d(\cdot, \cdot)$: For any two requests $\gamma = (f,b)$ and $\tilde \gamma = (\tilde f, \tilde b)$:
\vspace{-0.5em}
\begin{align*}
    d(\gamma, \tilde\gamma) = \sup_{x\in \X} \left| f(x) - \tilde f(x) \right| + \sup_{x \in \X} \left| b(x) - \tilde b(x) \right| \,.
\end{align*}

We will assume that $0 \in \X$. This allows the decision maker to avoid spending the resource if she so chooses and ensures feasibility. Moreover, let $\bar x = \max_{x \in \X} x$. We will make standard regularity assumptions \citep{jiang2020online, balseiro2022online, balseiro2022best}: there exist $\bar f, \bar b \in \R_+$ such that $f(x) \leq \bar f$ and $b(x) \leq \bar b$ for all $x \in \X$. Like \citet{jiang2020online}, we will also assume that there exists $\kappa \in \R_+$ such that $f(x) \leq \kappa \cdot b(x)$ for all $x \in \X$, i.e., the maximum rate of return from spending the resource is bounded above by $\kappa$.

At time $t \in [T]$, the following sequence of events takes place: (i) a request $\gamma_t = (f_t, b_t)$ arrives; (ii) the decision maker observes $\gamma_t$ and chooses an action $x_t \in \X$ based on the information observed so far; (iii) the request consumes $b_t(x_t)$ amount of the resource and generates a reward of $f_t(x_t)$. The decision maker aims to maximize her rewards subject to her budget constraint. A policy $\{x_t(\cdot)\}_t$ for the decision maker maps requests to actions $x_t: \S \to \X$ based on the available information at each time step, i.e., the action $x_t(\gamma_t)$ at time $t \in [T]$ can depend on the historical requests $\{\gamma_s\}_{s=1}^{t-1}$ and the current request $\gamma_t$, but not the future requests $\{\gamma_s\}_{s=t+1}^T$. Moreover,  a policy is said to be budget-feasible if it respects the budget constraint by ensuring $\sum_{t=1}^T b_t(x_t(\gamma_t)) \leq B$ for every sequence $\{\gamma_t\}_t$.

The request $\gamma_t$ at time $t$ is drawn from a distribution $\PP_t \in \Delta(\S)$ unknown to the decision maker, independently of the requests at other time steps. We only require the requests $\{\gamma_t\}_t$ to be independent and allow the distributions $\PP_t$ to vary arbitrarily across time. We will measure the performance of a policy against the fluid-optimal benchmark, which is defined as:
\begin{align*}
	\fluid(\{\PP_t\}_t) \coloneqq \noneurips{\quad} \text{max} \quad  & \sum_{t=1}^T \E[f_t(x_t(\gamma_t))]\\
	\text{s.t.} \quad & \sum_{t=1}^T \E[b_t(x_t(\gamma_t))] \leq B\\
		&x_t: \S \to \X \quad \forall\ t \in [T] \,.
\end{align*}

Another benchmark common in the literature on online resource allocation is the expected hindsight optimal solution, which is defined as $\E[\opt(\{\gamma_t\}_t)]$ for
\begin{align*}
	\opt(\{\gamma_t\}_t) \coloneqq   \max_{x \in \X^T}\ \sum_{t=1}^T f_t(x_t)   \text{ s.t. } \sum_{t=1}^T b_t(x_t) \leq B\,.
\end{align*}
It is well-known that $\fluid(\{\PP_t\}) \geq \E[\opt(\{\gamma_t\}_t)]$, which makes our benchmark the stronger one (we provide a proof in Appendix~\ref{appendix:fluid-relaxation} for completeness). Hence, our performance guarantees relative to the fluid-optimal benchmark also imply the same guarantees for the expected hindsight-optimal benchmark.

More concretely, we use $R(A | \{\gamma_t\}_t)$ to denote the total reward of a policy $A$ on the request sequence $\{\gamma_t\}_t$, and the performance of an algorithm is measured using its expected regret against the fluid-optimal reward:
\begin{align*}
	\regret(A) \coloneqq  \fluid(\{\PP_t\}_t) - \E \left[R(A|\{\gamma_t\}_t) \right]\,.
\end{align*}%

Now, if the distributions $\{\PP_t\}_t$ are unknown and arbitrary, and no other information about $\{\PP_t\}_t$ is available, then the requests $\{\gamma_t\}_t$ can be adversarial. This case has been addressed in \citet{balseiro2022best}, where the authors showed no policy can achieve sub-linear regret. In this work, we address the setting in which the decision maker has additional information in the form of historical samples. In particular, we focus on the setting where the decision maker has access to one independent sample $\tilde \gamma_t \sim \tilde \PP_t$ for each $t \in [T]$. We will assume that the $\{\tilde \gamma_t\}$ samples are independent of the request sequence $\{\gamma_t\}_t$ and $\{\tilde \PP_t\}$ are not known to the decision maker. We will show that when the sampling distributions $\{\tilde \PP_t\}_t$ are not too far from the actual distributions $\{\PP_t\}_t$, which is a minimal relaxation over the adversarial setting, it is possible to achieve sub-linear regret. We refer to the collection of samples $\{\tilde \gamma_t\}_t$ as a \emph{trace} and allow the actions of the decision-maker to depend on it. Throughout this paper, we will use $\{\tilde \gamma_t\}_t$ to denote the trace and $\{\gamma_t\}_t$ to denote the (random) sequence of requests on which the decision maker wishes to maximize reward.

\section{Warmup: Learning the Dual and Earning with It}\label{sec:dual}

First, let us focus on the simpler case when $\tilde \PP_t = \PP_t$ for all $t \in [T]$, i.e., the sampling distributions are the same as the request distributions. At first glance, it may appear that only having access to one sample from each of request distributions $\PP_t$ yields too little information to achieve near-optimal rewards. If one were to attempt to directly learn the optimal solution of $\fluid(\{\PP_t\}_t)$, this initial impression would be accurate because of the high-dimensional nature of the space of all possible solutions $\{x_t(\cdot)\}_t$. Fortunately, we do not need to learn this high-dimensional information and can instead leverage the structure of the problem: the dual space is just one-dimensional and thereby amenable to learning. More precisely, the dual function $D(\mu | \{\PP_t\}_t)$ of $\fluid(\{\PP_t\}_t)$ at dual variable $\mu \geq 0$ is given by

\vspace{-1em}
\small
\begin{align*}
	&\max_{\{x_t(\cdot)\}_t} \sum_{t=1}^T \E[f_t(x_t(\gamma_t))] + \mu\left( B -  \sum_{t=1}^T \E[b_t(x_t(\gamma_t))]  \right) \\
	=&\ \mu \cdot B + \sum_{t=1}^T \max_{x_t: \S \to \X} \E\left[ f_t(x_t(\gamma_t)) - \mu\cdot b_t(x_t(\gamma_t)) \right] \\
	=&\ \mu \cdot B + \sum_{t=1}^T \E\left[ \max_{x_t \in \X} \left\{ f_t(x_t) - \mu \cdot b_t(x_t) \right\} \right] \,.
\end{align*}
\normalsize
Throughout, we assume $\argmax_{x \in \X} \left\{ f(x) - \mu \cdot b(x)\right\}$ is non-empty for all requests $\gamma \in \S$ and dual solutions $\mu \geq 0$. If we treat the dual variable $\mu$ as the per-unit price of the resource, $\max_{x_t \in \X} \left\{f_t(x_t) - \mu \cdot b_t(x_t)\right\}$ captures the profit maximization problem. The following terminology would be helpful in working with the dual.

\begin{definition}\label{definition:profit-maximizing-decision}
	For a request $\gamma = (f,b)$ and dual variable $\mu \geq 0$, let $x^*(\gamma, \mu)$ be the optimal solution of $\max_{x \in \X} \left\{f(x) - \mu \cdot b(x)\right\}$ with the largest value of $f(x)$. If there are multiple such solutions, pick one which minimizes $b(x)$. Moreover, let $f^*(\mu) \coloneqq f(x^*(\gamma, \mu))$ and $b^*(\mu) \coloneqq b(x^*(\gamma, \mu))$ be the corresponding reward and resource consumption respectively.
\end{definition}

We denote $D(\mu|\{\PP_t\}_t) = \mu\cdot B + \sum_{t=1}^T \E[f_t^*(\mu) - \mu \cdot b_t^*(\mu)]$. Throughout this paper, we will repeatedly leverage weak duality, which is a central property of duals. We state the property here and refer the reader to any standard text on convex optimization (e.g., \citealt{bertsekas2009convex}) for a proof.

\begin{proposition}[Weak Duality]\label{prop:global-weak-duality}
	For all request distributions $\{\PP_t\}_t$ and dual variables $\mu \geq 0$, we have $D(\mu|\{\PP_t\}_t) \geq \fluid(\{\PP_t\}_t)$, i.e.,

    \vspace{-1em}
    \small
	\begin{align*}
		\sum_{t=1}^T \E[f_t^*(\mu)] \geq \fluid(\{\PP_t\}_t) - \mu \cdot \left( B - \sum_{t=1}^T \E[b_t^*(\mu)] \right)\,.
	\end{align*}
    \normalsize
\end{proposition}

Observe that $\sum_{t=1}^T \E[f_t^*(\mu)]$ is exactly the expected reward the decision maker would receive if she had an infinite budget and she took actions which maximized profit with $\mu$ being the per-unit price of the resource. Moreover, $B - \sum_{t=1}^T \E[b_t^*(\mu)]$ is the amount by which the decision maker would underspend her budget in expectation if she were to take actions using $\mu$ as the price. Suppose we can find a dual variable $\mu \geq 0$ that satisfies approximate complementary slackness, i.e., $\mu$ satisfies at least one of the following statements: (1) $\mu = 0$ and maximizing profit with $\mu$ as the per-unit price results in total expenditure less than the budget $B$ \noneurips{($\sum_{t=1}^T b_t^*(\mu) \leq B$)} with high probability; (2) $\mu > 0$ and maximizing profit with $\mu$ as the per-unit price results in total expenditure close to the budget $B$ \noneurips{($\sum_{t=1}^T b_t^*(\mu) \approx B$)} with high probability. Then, if the decision maker were to use $\mu$ as the price and make decisions to maximize profit, she will not run out of budget too early and the complementary slackness term $\mu \cdot \left(B - \sum_{t=1}^T \E[b_t^*(\mu)] \right)$ would also be small. Therefore, such a $\mu$ would yield rewards that are close to $\fluid(\{\PP_t\}_t)$, i.e., yield small regret, as required. 

We next describe how such a $\mu$ can be learned from the sample trace $\{\tilde \gamma_t\}$ when $\tilde \PP_t = \PP_t$ for all $t \in [T]$. We will assume that the distributions satisfy the following mild and standard assumption \citep{devanur2009adwords, agrawal2014dynamic} to exclude the degenerate case.

\begin{assumption}[General Position]\label{assumption:general-position}
	The request sequence $\{\gamma_t\}_t \sim \prod_t \PP_t$ is in general position almost surely: For any $\mu \geq 0$, there is at most one request with multiple profit maximizers, i.e.,
	\begin{align*}
		\left| \left\{ t \in [T] : |\argmax_{x \in \X} \left\{ f_t(x) - \mu \cdot b_t(x)\right\}| > 1  \right\} \right| \leq 1 \,.
	\end{align*}
	Moreover, the sample trace $\{\tilde\gamma_t\}_t \sim \prod_t \tilde \PP_t$ is also in general position almost surely.
\end{assumption}

Assumption~\ref{assumption:general-position} is made without any loss of generality because, as pointed out in \citet{devanur2009adwords} and \citet{agrawal2014dynamic}, adding an infinitesimally-small perturbation to the reward functions always results in perturbed distributions that satisfy Assumption~\ref{assumption:general-position} with only an infinitesimal change in the value of $\fluid(\{\PP_t\}_t)$ (see Appendix~\ref{appendix:general-position} for a formal description). Assumption~\ref{assumption:general-position} ensures that there exists a dual solution $\tilde \mu \geq 0$ which spends close to the budget $B$ on the trace $\{\tilde\gamma_t\}_t$ if it is possible to do so. In fact, as the following lemma shows, the optimal empirical dual solution satisfies this property. 

\begin{lemma}\label{lemma:trace-dual}
	Suppose the trace $\{\tilde \gamma\}_t \sim \prod_t \tilde \PP_t$ is in general position, and consider
	\begin{align*}
		\tilde \mu \in \argmin_{\mu \geq 0} \left\{\mu\cdot B + \sum_{t=1}^T \max_{x \in \X} \left\{\tilde f_t(x) - \mu \cdot \tilde b_t(x) \right\}\right\}\,.
	\end{align*}
	Then, at least one of the following statements holds:
	\begin{enumerate}
		\item $\tilde \mu = 0$ and $\sum_{t=1}^T \tilde b_t^*(\tilde \mu) \leq B + \bar b$.
		\item $\left|B - \sum_{t=1}^T \tilde b_t^*(\tilde \mu) \right| \leq \bar b$.
	\end{enumerate}
\end{lemma}

Recall that weak duality (Proposition~\ref{prop:global-weak-duality}) suggests that finding a dual solution which satisfies approximate complementary slackness with high probability would yield reward close to $\fluid(\{\PP_t\})$. Lemma~\ref{lemma:trace-dual} states that we can compute a dual variable $\tilde \mu$ which satisfies approximate complementary slackness on the trace. To finish the argument, we require a uniform convergence bound which shows that expenditure on the trace (or the sequence of requests) is concentrated close to the expected expenditure for all dual variables $\mu \geq 0$.

\begin{theorem}\label{thm:concentration}
	For $r(T) \coloneqq 8\bar b \cdot \sqrt{T \log(T)}$ and request distributions $\{\tilde \PP_t\}_t$, the following uniform convergence bound holds

    \vspace{-1em}
    \footnotesize
	\begin{align*}
		\Pr\left( \sup_{\mu \geq 0} \left| \sum_{t=1}^T \tilde b_t^*(\mu) - \sum_{t=1}^T \E_{\hat \gamma_t \sim \tilde \PP_t}\left[\hat b_t^*(\mu) \right] \right| \geq r(T) \right) \leq  \frac{1}{T^2} \,. 
	\end{align*}
    \normalsize
\end{theorem}

\begin{algorithm}[t!]
	\SetAlgoLined
	\textbf{Input:} Trace $\{\tilde \gamma_t\} \sim \prod_t \tilde \PP_t$, initial budget $B_1 = B$.\\
	\textbf{Compute an Optimal Empirical Dual Solution: }

    \vspace{-1em}
    \small
	\begin{align}\label{eq:empirical-dual}
		\tilde \mu \in \argmin_{\mu \geq 0} \left\{ \mu\cdot B + \sum_{t=1}^T \max_{x \in \X} \left\{\tilde f_t(x) - \mu \cdot \tilde b_t(x) \right\} \right\}
	\end{align}
    \normalsize
	\For{$t =1, \dots, T$}{
		Receive request $\gamma_t = (f_t, b_t) \sim \PP_t$.\\
		Make the primal decision $x_t$ and update the remaining resources $B_t$:

        \vspace{-1em}
        \small
		\begin{align}
		&{x}'_{t} \in \argmax_{x\in\X} \left\{f_{t}(x)- \tilde \mu \cdot b_{t}(x) \right\}\ ,\nonumber \\
		&x_{t}=\begin{cases}
			{x}'_t    & \text{  if } b_t({x}'_t)\le B_t  \\
		0
		& \text{ otherwise}
		\end{cases}\ , \nonumber \\
		&B_{t+1} = B_t - b_t(x_t) . \nonumber
		\end{align}
        \normalsize
	}
	\caption{Learning the Dual and Earning with It}
	\label{alg:learned-dual}
\end{algorithm}

With Theorem~\ref{thm:concentration} in hand, we are now ready to state and prove the regret guarantee for Algorithm~\ref{alg:learned-dual}. It first learns an empirical optimal dual variable $\tilde \mu$ from the trace $\{\tilde \gamma_t\}_t$, and then uses it as the per-unit price of the resource to take profit-maximizing actions on the request sequence $\{\gamma_t\}_t$.

\begin{theorem}\label{thm:learned-dual-regret}
	If $\PP_t = \tilde \PP_t$ for all $t \in [T]$, then Algorithm~\ref{alg:learned-dual} (denoted by $A$) satisfies $\regret(A) \leq  12 \kappa \bar b + 2 \kappa r(T)$.
\end{theorem}

\citet{arlotto2019uniformly} showed that every algorithm must incur a regret of $\Omega(\sqrt{T})$, even when the request distributions are identical (i.e., $\PP_t = \PP$ for all $t \in [T]$) and known to the decision-maker ahead of time. Thus, Theorem~\ref{thm:learned-dual-regret} shows the regret of Algorithm~\ref{alg:learned-dual} achieves a near-optimal dependence on $T$ with just a single sample per distribution, despite the request distributions being unknown and time-varying. However, as the following example demonstrates, this regret bound critically relies on the assumption that $\PP_t = \tilde \PP_t$ for all $t \in [T]$, and is fragile to even slight deviations from it. This fact was also demonstrated in \citet{jiang2020online} in a related context which inspired the following example.

\begin{example*}\label{example:dual-bad}
    Fix a small $\epsilon > 0$, an even horizon $T$ and budget $B = T/2$. Assume actions are accept/reject decisions, i.e., $\X = \{0,1\}$, and the reward/resource consumption functions are linear with $\coeff(b) = 1$ for all $\gamma = (f,b) \in \S$. In this setting, a request is completely determined by the coefficient $\coeff(f)$ of its reward function. We will overload notation and use $\gamma$ to denote this coefficient. Set $\tilde \PP_t = \text{Unif }([1 + \epsilon, 1 + 2 \epsilon])$ for all $t \leq T/2 + 1$ and $\tilde \PP_t = \text{Unif }([1 - \epsilon, 1])$ for all $t \geq T/2 + 2$. Moreover, set $\PP_t = \text{Unif }([1 - \epsilon, 1])$ for all $t \in [T]$. Then, it is easy to see that $\W(\PP_t, \tilde \PP_t) \leq 3 \epsilon$ for all $t \in [T]$. Also, observe that any trace $\{\tilde \gamma_t\}_t \sim \prod_t \tilde \PP_t$ would satisfy $\tilde\gamma_t \geq 1 + \epsilon$ for all $t \leq T/2 + 1$ and  $\tilde\gamma_t \leq 1$ for all $t \geq T/2 + 2$. Hence, we always have $\tilde \mu \geq 1 + \epsilon$. On the other hand, we also always have $\gamma_t \leq 1$ for all $t \in [T]$. Therefore, Algorithm~\ref{alg:learned-dual} sets $x_t' = 0$ for all $t \in [T]$, yielding a reward of 0. Whereas, $\fluid(\{\PP_t\}_t) \geq (1 - \epsilon) \cdot (T/2)$, thereby making the regret linear in $T$.
\end{example*}

Since $\epsilon > 0$ was arbitrary in the above example, it shows that even infinitesimally-small differences between the sampling and request distributions can lead to linear regret for Algorithm~\ref{alg:learned-dual}. This is antithetical to our goal of developing robust online algorithms for pacing. Formally, we would like to develop online algorithms that achieve regret which is small and degrades smoothly as $\sum_{t=1}^T \W(\PP_t, \tilde \PP_t)$ grows large. Nonetheless, although Algorithm~\ref{alg:learned-dual} falls short of this goal, it highlights the power of dual-based algorithms. Building on the intuition developed in this section, we next describe and analyze a Dual FTRL algorithm that achieves near-optimal regret while being robust to discrepancies between the sampling distributions $\{\tilde \PP_t\}_t$ and the request distributions $\{\PP_t\}_t$.

%
%
%
%
%
%
%
%

\section{Dual FTRL with Target Rate Estimation}\label{sec:descent}

In this section, we will develop an algorithm based on Dual Follow-The-Regularized-Leader (FTRL) that achieves near-optimal regret with a single trace, and is robust to discrepancies between the sampling distributions and request distributions. Now, if one had complete knowledge of the sampling distributions $\{\tilde \PP_t\}_t$, then one can solve $\fluid(\{\tilde \PP_t\}_t)$ to find an optimal solution and run Dual Gradient Descent with the goal of spending the same as the optimal solution at each time step. It is known from \citet{jiang2020online} that this approach achieves $O(\max\{\sqrt{T}, \sum_{t=1}^T \W (\tilde \PP_t, \PP_t)\})$ regret, thereby making it rate optimal and robust to discrepancies. However, with just a single sample from each of distributions $\tilde \PP_t$, we are far from having complete knowledge of $\{\tilde \PP_t\}_t$. Despite this apparent lack of data, a careful analysis of Dual FTRL will allow us to show that it achieves near-optimal regret rate in a robust manner.

\subsection{Dual Follow-The-Regularized-Leader} 

\begin{algorithm}[t!]
	\SetAlgoLined
	{\bf Input:} Initial resource endowment $B_1 = B$, target consumption sequence $\{\lambda_t\}_{t=1}^T$, regularizer $h: \R \rightarrow \R$ and step-size $\eta$. \\
    Set initial dual solution $\mu_1 = \argmin_{\mu \in [0,\kappa]} h(\mu)$.\\
	\For{$t=1,\ldots,T$}{
		Receive request $\gamma_t = (f_t, b_t) \sim \PP_t$.\\
		Make the primal decision $x_t$ and update the remaining resources $B_t$:
		\begin{align}
		&{x}'_{t} \in \argmax_{x\in\X_t}\left\{f_{t}(x)-\mu_{t} \cdot b_{t}(x)\right\} \ ,\label{eq:primal_decision} \\
		&x_{t}=\begin{cases}
			{x}'_t    & \text{  if } b_t({x}'_t)\le B_t  \\
		0
		& \text{ otherwise}
		\end{cases}\ , \nonumber \\
		&B_{t+1} = B_t - b_t(x_t) . \nonumber
		\end{align}
		Obtain a sample sub-gradient of the dual function $D(\mu|\PP_t, \lambda_t)$: $g_t = \lambda_t -b_t(x_t')$.
  
		Update the dual iterate with FTRL:
        \begin{align}\label{eq:FTRL}
            &\mu_{t+1} = \argmin_{\mu \in [0,\kappa]} \left\{ \eta \sum_{r=1}^t g_r \cdot \mu + h(\mu) \right\}  \,,
        \end{align}
	}
	\caption{Dual Follow-The-Regularized-Leader}
	\label{alg:dual-descent}
\end{algorithm}

The non-stationarity of the request distributions necessitates the need for Dual FTRL that can incorporate target resource consumptions (Algorithm~\ref{alg:dual-descent}). It takes as input a target sequence $\{\lambda_t\}_{t=1}^T$ which specifies $\lambda_t \geq 0$ to be the amount of resource Dual FTRL should attempt to consume at time $t$. Moreover, like FTRL \citep{shalev2012online, hazan2016introduction}, it also takes as input a regularizer $h(\cdot)$, an initial dual variable $\mu_1$ and a step-size $\eta$. We will make the standard assumption that the regularizer $h(\cdot)$ is differentiable
and is $\sigma$-strongly convex in the $\|\cdot \|_1$ norm.

Before stating the performance bound of Algorithm~\ref{alg:dual-descent}, we introduce some preliminaries. Given a budget of $\beta_t$ for period $t \in [T]$, the optimal expected reward which can be collected in period $t$ is captured by the following fluid optimization problem:
\begin{align*}
	\fluid(\PP_t, \beta_t) \coloneqq \text{max} \quad  & \E[f_t(x_t(\gamma_t)]\\
	\text{s.t.} \quad & \E[b_t(x_t(\gamma_t))] \leq \beta_t\\
		&x_t: \S \to \X \,.
\end{align*}

The dual function of $\fluid(\PP_t, \beta_t)$ is given by
\begin{align*}
	D(\mu|\PP_t, \beta_t) \coloneqq \mu \cdot \beta_t + \E\left[\max_{x \in \X} \left\{f_t(x) - \mu \cdot b_t(x) \right\} \right]\,,
\end{align*}
for any $\mu \geq 0$. Then, by weak duality, we have $\fluid(\PP_t, \beta_t) \leq D(\mu| \PP_t, \beta_t)$ for all $\mu \geq 0$. Moreover, since dual functions are always convex (they are the suprema of linear functions), the dual function $D(\cdot|\PP_t, \beta_t)$ is convex.

Theorem~\ref{thm:dual-descent-regret} states a general regret bound for Algorithm~\ref{alg:dual-descent} with an arbitrary target sequence $\{\lambda_t\}_t$ and against a general benchmark $\sum_{t=1}^T D(\mu_t|\PP_t, \beta_t)$. Since $\fluid(\PP_t, \beta_t) \leq D(\mu| \PP_t, \beta_t)$ by weak duality, Theorem~\ref{thm:dual-descent-regret} also characterizes the performance against the weaker benchmark $\sum_{t=1}^T \fluid(\PP_t, \beta_t)$, which is simply the optimal expected reward the decision maker would collect if she spent $\beta_t$ at time $t$.


\begin{theorem}\label{thm:dual-descent-regret}
	Consider Algorithm~\ref{alg:dual-descent} with target consumption sequence $\{\lambda_t\}_t$, regularizer $h(\cdot)$ and step-size $\eta$. Then, for a benchmark sequence $\{\beta_t\}_t$, we have

    \vspace{-1em}
    \small
	\begin{align*}
		\E\left[ \left\{ \sum_{t=1}^T D(\mu_t|\PP_t, \beta_t) \right\} - R(A|\{\gamma_t\}_t) \right] \leq\ R_1 + R_2 + R_3 \,,
	\end{align*}
    \normalsize
	where
 
    \vspace{-0.8em}
    \small
	\begin{itemize}
		\item $R_1 = \kappa \bar b + \frac{2(\bar b+ \bar \lambda)^2}{\sigma} \cdot \eta T + \frac{d_R}{\eta}$, for $\bar \lambda = \max_t\lambda_t$ and $d_R = \max\{h(0) - h(\mu_1), h(\kappa) - h(\mu_1)\}$.
		\item $R_2 = \kappa \cdot \left( \left\{ \sum_{t=1}^T \lambda_t \right\} - B \right)^+$,
		\item $R_3 = \E\left[ \sum_{t=1}^T \mu_t \cdot (\beta_t - \lambda_t) \right]$.
	\end{itemize}
 \normalsize
\end{theorem}

Theorem~\ref{thm:dual-descent-regret} decomposes the regret of Algorithm~\ref{alg:dual-descent} into three terms, where (i) $R_1$ is simply the regret associated with the FTRL algorithm in the OCO setting \citep{hazan2016introduction}; (ii) $R_2$ captures the overspending error, which is large whenever the total target consumption $\sum_{t=1}^T \lambda_t$ is in excess of the budget $B$; (iii) $R_3$ captures the underestimation error, which is a weighted sum over the amounts by which the target sequence $\{\lambda_t\}_t$ underestimates the benchmark sequence $\{\beta_t\}_t$, with weights equal to the dual iterates $\mu_t$. Observe that there is an inherent tension between the overspending error $R_2$ and the underestimation error $R_3$---$R_2$ can be made smaller by making the target consumptions $\{\lambda_t\}_t$ smaller, but this in turn makes $R_3$ bigger, and vice versa. To obtain the desired performance guarantees for Algorithm~\ref{alg:dual-descent} (see Theorem~\ref{thm:main-result}), we need to carefully choose the benchmark sequence $\{\beta_t\}_t$ and the target sequence $\{\lambda_t\}_t$, which is what we do next (see \eqref{eq:benchmark-target-def}). We go on to show that
\begin{itemize}
    \item Our choice of target sequence does not overspend too much. In particular, it satisfies $R_2 \leq \kappa \cdot \bar b$ (see \eqref{eq:R2}).
    \item Our choice of benchmark sequence $\{\beta_t\}_t$ ensures
    \small
        \begin{align*}
			&\fluid(\{\PP_t\}_t) - \sum_{t=1}^T D(\mu_t| \PP_t, \beta_t)\\ 
   \leq&\ \tilde O\left(\max\left\{\sqrt{T}, \sum_{t=1}^T \W(\PP_t, \tilde\PP_t) \right\} \right)
		\end{align*}
    \normalsize
        i.e., the benchmark in Theorem~\ref{thm:dual-descent-regret} is at most $\tilde O\left(\max\left\{\sqrt{T}, \sum_{t=1}^T \W(\PP_t, \tilde\PP_t) \right\} \right)$ larger than our desired benchmark $\fluid(\{\PP_t\}_t)$ (see Lemma~\ref{lemma:discrepancy-term} and the discussion that follows).
    
    \item Moreover, our choice of the sequences in combination with an intricate argument, consisting of a coupling lemma and a leave-one-out thought experiment, allows use to prove $R_3 = O(\sqrt{T})$ (see Subsection~\ref{subsec:R3}). 
\end{itemize}

Finally in Subsection~\ref{subsec:putting-together}, we combine everything to prove the desired regret bound for Algorithm~\ref{alg:dual-descent}.

\subsection{Choosing the Target and Benchmark Sequences}

We define the target and benchmark sequences using the empirical optimal dual solution computed from the trace $\{\tilde \gamma_t\}_t$:
\begin{align*}
	\tilde \mu \in \argmin_{\mu \geq 0} \left\{ \mu\cdot B + \sum_{t=1}^T \max_{x \in \X} \left\{ \tilde f_t(x) - \mu \cdot \tilde b_t(x) \right\} \right\}\,.
\end{align*}
If there are multiple minimizers, set $\tilde \mu$ to be the smallest one. Given the empirical optimal dual solution $\tilde \mu$, the target and benchmark sequences are defined as 
\begin{align}\label{eq:benchmark-target-def}
	\beta_t = \E_{\hat \gamma_t \sim \tilde \PP_t} \left[ \hat b_t^*(\tilde\mu) \right] \quad \text{ and } \quad \lambda_t = \tilde b_t^*(\tilde \mu) \,,
\end{align}
where $\hat \gamma_t = (\hat f_t, \hat b_t)$. In other words, the benchmark sequence is the expected consumption and the target sequence is the empirical consumption on the trace if we were to make profit-maximizing decisions using the empirical optimal dual solution $\tilde \mu$ as the price of the resource. Instead of learning the empirical optimal dual $\tilde \mu$ and directly making decisions with it like we did in Algorithm~\ref{alg:learned-dual}, we use $\tilde \mu$ to learn the empirical consumptions $\{\lambda_t\}_t$ and use Algorithm~\ref{alg:dual-descent} to track this target. Unlike the former, we will show that the latter approach is robust to discrepancies between the sampling and request distributions, while maintaining the same $\tilde O(\sqrt{T})$-regret guarantee. Importantly, note that the benchmark sequence $\{\beta_t\}_t$ cannot be computed in practice because it requires full knowledge of the request distributions. Algorithm~\ref{alg:dual-descent} respects this limitation and does not require knowledge of the benchmark sequence $\{\beta_t\}_t$; we only use it for our analysis. 

Our choice of $\{\lambda_t\}_t$ and Lemma~\ref{lemma:trace-dual} immediately imply
\vspace{-0.7 em}
\begin{align}\label{eq:R2}
	R_2 = \kappa \cdot \left( \left\{ \sum_{t=1}^T \lambda_t \right\} - B \right)^+\leq \kappa\cdot \bar b\,.
\end{align} 

Next, we show that, for our choice of the benchmark sequence $\{\beta_t\}_t$, the benchmark $\sum_{t=1}^T D(\mu_t|\PP_t, \beta_t)$ of Theorem~\ref{thm:dual-descent-regret} is not too far from the desired benchmark $\fluid(\{\PP_t\})$.

\begin{lemma}\label{lemma:discrepancy-term}
	For any dual variable $\tilde \mu \geq 0$, dual iterates $\{\mu_t\}_t \in [0,\kappa]^T$ and benchmark sequence $\{\beta\}_t$ with $\beta_t = \E_{\hat \gamma \sim \tilde \PP_t}[ \hat b^*(\tilde \mu)]$, we have
 
    \vspace{-0.7 em}
    \small
	\begin{align*}
		\sum_{t=1}^T D(\mu_t | \PP_t, \beta_t) \geq &\fluid(\{\PP_t\}_t) - \tilde \mu \cdot \left(B - \sum_{t=1}^T \beta_t \right)\\
        &- 2(1 + \kappa) \cdot \sum_{t=1}^T \W(\PP_t, \tilde \PP_t) \,.  
	\end{align*}
 \normalsize
\end{lemma}
\vspace{-0.7em}

Observe that Theorem~\ref{thm:concentration} implies that, with probability at least $1 - 1/T^2$, we have $\left| \sum_{t=1}^T \beta_t - \sum_{t=1}^T \lambda_t \right| \leq r(T)$. Combining this with Lemma~\ref{lemma:trace-dual} yields
\begin{align}\label{eq:discrepancy-term}
	\tilde \mu \cdot \left(B - \sum_{t=1}^T \beta_t \right) &\leq \tilde \mu \cdot \left(r(T) + B - \sum_{t=1}^T \lambda_t \right) \nonumber \\
	 &= \tilde \mu \cdot r(T) + \tilde \mu \cdot \left(B - \sum_{t=1}^T \lambda_t \right) \nonumber \\
	 &\leq  \tilde \mu \cdot \left(r(T) + \bar b \right) \nonumber \\
	 &\leq \kappa \cdot r(T) + \kappa \bar b\,,
\end{align}

\vspace{-0.7em}thereby showing that the benchmark $\sum_{t=1}^T D(\mu_t|\PP_t, \beta_t)$ of Theorem~\ref{thm:dual-descent-regret} is not too far from the desired benchmark $\fluid(\{\PP_t\})$. In order to establish the desired regret and robustness guarantees for Algorithm~\ref{alg:dual-descent}, all that remains to show is that $R_3 \leq \tilde O( \sqrt{T})$. However, as we demonstrate in the next subsection, this step is rife with challenges.

\subsection{Bounding $\mathbf{R_3}$}\label{subsec:R3}

We begin with a brief discussion of the challenges involved in bounding $R_3$. It is illuminating to consider the slightly more permissive setting in which the decision maker has access to two sample traces: suppose in addition to trace $\{\tilde\gamma_t\}_t \sim \prod_t \tilde \PP_t$, we had access to an additional trace $\{\dbtilde \gamma_t\}_t \sim \prod_t \tilde\PP_t$. Then, we could compute $\tilde \mu$ using $\{\dbtilde \gamma_t\}_t$ as follows
\begin{align*}
	\tilde \mu \in \argmin_{\mu \geq 0} \left\{\mu\cdot B + \sum_{t=1}^T \max_{x \in \X} \left\{ \dbtilde f_t(x) - \mu \cdot \dbtilde b_t(x) \right\} \right\}\,,
\end{align*}
making it completely independent of $\{\tilde \gamma_t\}_t$. With this modified $\tilde \mu$, we continue to define $\{\beta_t\}_t, \{\gamma_t\}_t$ as before (see~\eqref{eq:benchmark-target-def}). As a consequence, we get that $\mu_s$ is completely determined by $\{\gamma_t\}_{t=1}^{s-1}$ and $\{\lambda_t\}_{t=1}^{s-1}$, with the latter being completely determined by $\tilde \mu$ and $\{\tilde \gamma_t\}_{t=1}^{s-1}$. This makes $\mu_s$ independent of $\lambda_s$ conditional on $\tilde \mu$, and consequently yields

\vspace{-1em}
\footnotesize
\begin{align*}
	\E\left[ \mu_s\cdot (\beta_s - \lambda_s) \big|\ \tilde \mu, \{\tilde \gamma_t, \gamma_t\}_{t=1}^{s-1} \right] &= \mu_s \cdot \left( \beta_s - \E\left[\tilde b_s^*(\tilde \mu) \big|\ \tilde \mu\right] \right)\\
    &= 0\,.
\end{align*}
\normalsize
Thus, we can apply the Tower Rule of conditional expectations to get
\small
\begin{align*}
	R_3 &= \sum_{s=1}^T \E\left[\mu_s \cdot (\beta_s - \lambda_s) \right]\\
 &= \sum_{s=1}^T \E\left[ \E\left[ \mu_s\cdot (\beta_s - \lambda_s) \big|\ \tilde \mu, \{\tilde \gamma_t, \gamma_t\}_{t=1}^{s-1} \right] \right]\\
 &= 0\,.
\end{align*}
\normalsize

\vspace{-1 em}It is straightforward to see that the bounds on $R_1$ and $R_2$ established in the previous subsection continue to hold in this two-trace setting. Therefore, two traces allow us to achieve the near-optimal $\tilde O(\sqrt{T})$-regret while being robust to discrepancies between $\tilde\PP_t$ and $\PP_t$.

Although moving from two traces to one trace might appear to be a minor change, it introduces correlations that make the proof much more difficult. Observe that Algorithm~\ref{alg:dual-descent} determines $\mu_s$ using $\{\lambda_t\}_{t=1}^{s-1}$, all of which depend on $\tilde \mu$, which in turn is computed using the request $\tilde \gamma_s$. Furthermore, $\lambda_s$ directly depends on $\tilde \gamma_s$. Thus, $\mu_s$ and $\lambda_s$ are intricately correlated with each other, which breaks the aforementioned argument for the two-trace setting. Nonetheless, $R_3$ can still be shown to be small, as we note in the following lemma and prove in the remainder of this subsection.
\begin{lemma}\label{lemma:R3}
	For all $s \in [T]$, we have
    \vspace{-0.5em}
	\begin{align*}
		R_3 = \sum_{s=1}^T \E \left[ \mu_s \cdot (\beta_s - \lambda_s) \right] \leq \frac{4 \eta \bar b^2}{\sigma} \cdot T\quad\,.
	\end{align*}
\end{lemma}
\vspace{-0.5em}

We prove Lemma~\ref{lemma:R3} in the remainder of this subsection. The following lemma will find repeated use in the proof. In keeping with economic intuition, it shows that increasing the price (dual variable) leads to smaller consumption under the profit-maximzing decision.
 
\begin{lemma}[Monotonicity]\label{lemma:monotonicity}
	For $\mu > \mu'$, request $\gamma = (f,b) \in \S$, $x \in \argmax_{z \in \X} \{f(z) - \mu \cdot b(z)\}$ and $x' \in \argmax_{z \in \X} \{f(z) - \mu' \cdot b(z)\}$, we have $b(x) \leq b(x')$.
\end{lemma}

Fix an $s \in [T]$. We will get around the correlation between $\mu_s$ and $\tilde \gamma_s$ by conducting the following leave-one-out thought experiment: suppose we remove the $s$-th sample $\tilde \gamma_s$, compute $\tilde \mu$ on the remaining trace $\{\tilde \gamma_t\}_{t \neq s}$, and run Algorithm~\ref{alg:dual-descent} with the resulting target sequence. More precisely, in this thought experiment, we set $\tilde \PP_s$ to be the distribution which always serves the request $\gamma = (f,b)$ with $f(x) = b(x) = 0$ for all $x \in \X$. Thus, $\tilde f_s(x)= \tilde b_s(x) = \tilde f^*(\mu) = \tilde b^*(\mu) = 0$ for all $x \in \X$ and $\mu \geq 0$. We will use the superscript $(-s)$ to denote the various variables in this thought experiment:
\vspace{-0.5em}
\begin{itemize}
	\item $\tilde \mu^{(-s)}  \in \argmin_{\mu \geq 0} \mu\cdot B + \sum_{t\neq s} \max_{x \in \X} \{\tilde f_t(x) - \mu \cdot \tilde b_t(x) \}$. If there are multiple minimizers, set $\tilde\mu^{(-s)}$ to be the smallest one amongst them.
	\item $\lambda_t^{(-s)} = \tilde b_t^* \left( \tilde \mu^{(-s)} \right)$ for all $t \in [T]$.
	\item $\mu_t^{(-s)}$ is the $t$-th iterate of Algorithm~\ref{alg:dual-descent} with the target consumption sequence $\left\{ \lambda_t^{(-s)} \right\}_t$.
\end{itemize}
\vspace{-0.5em}

We begin by characterizing the impact of this change on the target consumption sequence.

\begin{lemma}\label{lemma:change-in-target}
	For every sample trace $\{\tilde \gamma_t\}_t$, we have $\tilde \mu \geq \tilde \mu^{(-s)}$ and $\lambda_t \leq \lambda_t^{(-s)}$ for all $t \neq s$. Moreover, $\sum_{t=1}^{s-1} \left|\lambda_t^{(-s)} - \lambda_t \right| \leq 3 \bar b \,.$
\end{lemma}

Lemma~\ref{lemma:change-in-target} shows that the target sequences $\{\lambda_t\}_t$ and $\{\lambda^{(-s)}_t\}_t$ are close to each other. Next, we couple the dual iterates $\mu_t$ and $\mu_t^{(-s)}$ generated by Algorithm~\ref{alg:dual-descent} to show that they never stray too far from each other whenever the target sequences are close.

\begin{lemma}[Dual Iterate Coupling]\label{lemma:coupling}
	Let $\{\mu_t\}_t$ and $\{\mu_t'\}_t$ denote the iterates generated by Algorithm~\ref{alg:dual-descent} on the request sequence $\{\gamma_t\}_t$ for the target sequences $\{\lambda_t\}_t$ and $\{\lambda'_t\}_t$ respectively. Assume that the initial iterates are the same, i.e., $\mu_1 = \mu_1'$. Then, for all $s \in [T]$, we have
	\begin{align*}
		\left| \mu_s  - \mu'_s \right| \leq \frac{\eta}{\sigma} \cdot \left\{ \sum_{t=1}^{s-1} |\lambda_t - \lambda_t'| \right\} + \frac{\eta}{\sigma} \cdot \bar b\,.
	\end{align*}
\end{lemma}

\vspace{-0.7em}Applying Lemma~\ref{lemma:coupling} with $\lambda_t' = \lambda_t^{(-s)}$ and using Lemma~\ref{lemma:change-in-target} yields $\left| \mu_s  - \mu'_s \right| \leq \frac{\eta}{\sigma} \cdot \left\{ 3\bar b \right\} + \frac{\eta}{\sigma} \cdot \bar b = \frac{4 \eta \bar b}{\sigma} \,.$

Combining this with the fact that $|\beta_s - \lambda_s| \leq \bar b$, we get
\tiny
\begin{align}\label{eq:lemma-R3-inter}
    \E \left[ \mu_s \cdot (\beta_s - \lambda_s) \right] &= \E \left[ \left(\mu_s - \mu_s^{(-s)} \right) \cdot (\beta_s - \lambda_s) \right] + \E \left[ \mu_s^{(-s)} \cdot (\beta_s - \lambda_s) \right] \nonumber \\
    &\leq \frac{4 \eta \bar b}{\sigma} \cdot \bar b + \E \left[ \mu_s^{(-s)} \cdot (\beta_s - \lambda_s) \right]\,.
\end{align}
\normalsize

\vspace{-1em}The next lemma shows that the second term is non-positive. Its proof critically leverages the fact that the iterate $\mu_s^{(-s)}$ is independent of the $s$-th sample in the trace $\tilde \gamma_s$ (which is used to determine $\lambda_s$). This is in stark contrast to $\mu_s$ which depends on $\tilde \gamma_s$, and demonstrates the merit of our leave-one-out thought experiment.

\begin{lemma}\label{lemma:leave-one-out}
	$\E \left[ \mu_s^{(-s)} \cdot (\beta_s - \lambda_s) \right] \leq 0$ for all $s \in [T]$.
\end{lemma}

Lemma~\ref{lemma:leave-one-out} in combination with \eqref{eq:lemma-R3-inter} yields $\E \left[ \mu_s \cdot (\beta_s - \lambda_s) \right] \leq 4 \eta \bar b^2/\sigma$. Summing over all $s \in [T]$ finishes the proof of Lemma~\ref{lemma:R3}.

\subsection{Putting It All Together}\label{subsec:putting-together}

In the previous subsections, we bounded $R_1$, $R_2$ and $R_3$, and related the benchmark $\sum_{t=1}^T D(\mu_t|\PP_t, \beta_t)$ from Theorem~\ref{thm:dual-descent-regret} to our desired benchmark $\fluid(\{\PP_t\}_t)$. Combining everything yields the following performance gaurantee for Algorithm~\ref{alg:dual-descent}.

\begin{theorem}\label{thm:main-result}
    Let $A$ be Algorithm~\ref{alg:dual-descent} with target sequence $\{\lambda_t\}_t$, where $\lambda_t = \tilde b_t^*(\tilde \mu)$ (as defined in \eqref{eq:benchmark-target-def}), regularizer $h(\cdot)$ and step-size $\eta = \sqrt{d_R/ T}$, where $d_R = \max\{h(0) - h(\mu_1), h(\kappa) - h(\mu_1)\}$. Then,
    
    \vspace{-1em}
    \small
    \begin{align*}
        \regret(A) \leq C_1 \sqrt{T \log(T)} + C_2 \sum_{t=1}^T \W(\PP_t, \tilde \PP_t)\,. 
    \end{align*}
    \normalsize
    where $C_1 = \frac{12 \bar b^2 \sqrt{d_R}}{\sigma} +  \sqrt{d_R} + 12 \kappa \bar b$ and $C_2 = 2(1+\kappa)$.
\end{theorem}

Observe that the regret of Dual FTRL satisfies $\regret(A) = \tilde O(\sqrt{T})$ whenever $\sum_{t=1}^T \W(\PP_t, \tilde \PP_t) = \tilde O(\sqrt{T})$. In other words, Dual FTRL achieves near-optimal regret with a single trace as long as the total discrepancy $\sum_{t=1}^T \W(\PP_t, \tilde \PP_t)$ is not too large. Finally, we would also like to note that our algorithm is extremely efficient computationally. In particular, due to the equivalence of FTRL and ``Lazy" Online Mirror Descent (OMD) (see \citealt{hazan2016introduction}), each dual update in \eqref{eq:FTRL} can be computed in constant time by running Lazy OMD. Moreover, given a trace $\{\tilde \gamma_t\}$ which is sorted in increasing order of bang-per-buck $\coeff(\tilde f_t)/\coeff(\tilde b_t)$, the target sequence $\{\lambda_t\}_t$ can be computed in $O(T)$ steps (see Appendix~\ref{appendix:fast-implmentation} for details).

\singlespace

\bibliographystyle{plainnat}
\bibliography{refs}

\appendix

\newpage
\pagenumbering{arabic}\renewcommand{\thepage}{ec \arabic{page}}

\begin{centering}
\LARGE
Electronic Companion:\\[1em]
Robust Budget Pacing with a Single Sample\\[1em]
\large
\end{centering}
\renewcommand{\theequation}{\thesection-\arabic{equation}}

\setstretch{1.5}
\section{Related Work}\label{sec:related-work}

\citet{balseiro2019learning} study budget pacing in repeated second-price auctions when the values and competing bids are either i.i.d. according to some unknown distribution or adversarially selected. They show that Dual Gradient Descent with the constant target $\lambda_t = B/T$ for all $t \in [T]$ attains the optimal regret of $O(\sqrt{T})$ in the i.i.d. stochastic setting, and the optimal parameter-dependent asymptotic competitive ratio (equal to ratio of the per-period budget to the maximum value) in the adversarial setting. \citet{zhou2008budget} also study the adversarial setting and provide a pacing-based algorithm that achieves a differently-parameterized competitive ratio which scales as the logarithm of the ratio of the highest-to-lowest return-on-investment, and show that it is optimal. \citet{kumar2022optimal} study an episodic setting and provide a density-estimation-based algorithm for learning the target expenditures for each episode. \citet{gaitonde2022budget} study the performance of the algorithm of \citet{balseiro2019learning} for the different objective of value maximization, and against the different benchmark comprised of pacing multipliers which spend the same amount $B/T$ at each time period. Under the no-overbidding assumption, they show that the algorithm of \citet{balseiro2019learning} achieves $O(T^{3/4})$ regret. Recent years have also seen significant attention being given to pacing in multi-buyer settings, but we focus on the single-agent setting here and refer the reader to the recent works of \citet{chen2021complexity} and \citet{gaitonde2022budget} for an overview.

More generally, budget pacing in second-price auctions is a special case of online linear packing, which in turn is a special case of the online resource allocation problem. Both these problems allow for multiple resources and have been studied extensively; we only provide a broad overview here. For the most part, these problems have also been studied in the i.i.d. stochastic model, or the slightly more general random arrival model (requests are selected by an adversary but arrive in a uniformly random order). \citet{devanur2009adwords} and \citet{feldman2010online} study online linear packing under the random arrival model, and show that learning the dual from the initial requests and then using it to make decisions yields $O(T^{2/3})$ regret. \citet{agrawal2014dynamic} extended these results to show that repeatedly solving for the dual at geometrically increasing intervals yields the optimal $O(\sqrt{T})$ regret. \citet{devanur2011near} and \citet{kesselheim2014primal} also achieve $O(\sqrt{T})$ regret but with a better dependence on the constants and the number of resources. \citet{gupta2016experts} give a dual-descent-based algorithm that also achieves $O(\sqrt{T})$ regret. 

\citet{agrawal2014fast} study online resource allocation with concave rewards and convex constraints, and give a dual-descent-based algorithm that achieves $O(\sqrt{T})$ regret. \citet{balseiro2022best} give a Dual Mirror Descent algorithm which attempts to spend $\lambda_t = B/T$ at each time step and show that it achieves $O(\sqrt{T})$ regret for the general online allocation problem. Their results also hold for stochastic models that are close to i.i.d. like periodic, ergodic etc. \citet{balseiro2022online} present and analyze a generalization of Dual Mirror Descent that can incorporate arbitrary target expenditures $\{\lambda_t\}_t$. When the horizon $T$ is not known exactly and only assumed to lie in some known uncertainty window, they provide a procedure for computing target expenditures which yield a near-optimal asymptotic competitive ratio. Our Dual FTRL algorithm (Algorithm~\ref{alg:dual-descent}) is equivalent to the Lazy-update version of the algorithm of \citet{balseiro2022online}. Importantly, with the exception of \citet{devanur2011near}, none of the aforementioned works provide algorithms which can achieve vanishing regret for the setting with time-varying distributions. \citet{devanur2011near} achieve $O(\sqrt{T})$ regret when the optimal expected reward for each distribution is known in advance. However, this quantity cannot be computed with a single sample for non-trivial distributions, and they do not provide guarantees for our sample-access setting.

Another line of work develops algorithms that beat $O(\sqrt{T})$ regret when the problem instance is well-structured. With the exception of \citet{banerjee2020constant} and \citet{banerjee2020uniform}, all of these works assume complete knowledge of the distributions and/or assume that the distributions are identical. When the number of requests of each type satisfies a concentration property between the trace and the actual requests, \citet{banerjee2020constant} and \citet{banerjee2020uniform} show that a constant regret can be achieved for online resource allocation using one sample per distribution. For the budget pacing problem, a type corresponds to a value and competing bid pair. Since complex machine-learning models are typically used to estimate advertiser values to a high precision, this translates to an extremely large number of possible types. Far from concentrating, these large number of types imply that one is unlikely to even observe a type more than once, making their primal-based method ineffective for budget pacing. Moreover, neither \citet{banerjee2020constant} nor \citet{banerjee2020uniform} do not provide any robustness guarantees for possible discrepancies between the sampling and true distributions, and their algorithm requires knowledge of the competing bid. Finally, our results are meaningful when the budget is much larger than the maximum amount one can spend on an auction/request, as is the case for budget pacing. In contrast, the literature on prophet inequalities considers a unit-cost variable-reward online allocation problem where the budget is only large enough to accept one request. See \citet{azar2014prophet, correa2019prophet, rubinstein2020optimal, caramanis2022single} for a sample-driven treatment of prophet inequalities.
\section{Application to Budget Pacing}\label{sec:application}

Here we discuss how the budget pacing problem fits as a special case of the online resource allocation problem that we study in this paper. Consider the setting in which a budget-constrained advertiser repeatedly participates in $T$ second-price auctions. For simplicity, assume that all ties are broken in favor of this advertiser. Let $v_t$ and $d_t$ denote her value and the highest competing bid in the $t$-th auction respectively. Moreover, let $B$ denote her budget, which represents the maximum amount she is willing to spend over all $T$ auctions.

We will assume that the tuple $(v_t, d_t)$ is drawn from some distribution $\PP_t$, independently of all other auctions. Now, observe that every bid of the advertiser results in one of two possible outcomes: (i) she bids greater than or equal to $d_t$, wins the auction, gains utility $v_t - d_t$ and pays $d_t$; (ii) she bids strictly less than $d_t$, loses the auction, gains zero value and pays zero. Thus, corresponding to the tuple $(v_t, d_t)$, we can define a corresponding request $\gamma_t$ with linear reward function $f_t(x) = (v_t - d_t) \cdot x$ and linear consumption function $b_t(x) = d_t \cdot x$, for the action space $x \in \{0,1\} = \{\text{lose}, \text{win}\}$. Similarly, corresponding to the sample trace of tuples $\{(\tilde v_t, \tilde d_t)\}_t$, we can define a trace $\{\tilde \gamma_t\}_t$ for the online allocation problem. This defines a corresponding instance of the online allocation problem. Since every bid either results in either a win or loss, the maximum expected utility (value - payment) that the advertiser can earn subject to her budget constraint is bounded above by $\fluid(\{\PP_t\}_t)$ for this instance. Finally, consider step $t$ of Algorithm~\ref{alg:dual-descent} on this instance. The decision $x_t$ is calculated as $x_{t} \in \argmax_{x\in\X_t}\left\{f_{t}(x)-\mu_{t} \cdot b_{t}(x)\right\}$. Therefore, $x_t = 1$ if $v_t - d_t \geq \mu d_t$, or equivalently $v_t/(1 + \mu_t) \geq d_t$, and $x_t = 0$ otherwise. Observe that, in a second price auction, if the advertiser bids $v_t/(1 + \mu_t)$, she will win ($x_t = 1$) if $v_t/(1 + \mu_t) \geq d_t$ and lose ($x_t = 0$) otherwise. Thus, by bidding $v_t/(1 + \mu_t)$, she can simulate the actions of Algorithm~\ref{alg:dual-descent} for the online allocation instance. Moreover, she does not require knowledge of the competing bid $d_t$ to compute her bid, which is crucial because $d_t$ is not known in practice. Once the auction is over, the expenditure $b_t(x_t) = d_t \cdot x_t$ is revealed to the advertiser. She can then use it to update the dual iterate according to \eqref{eq:FTRL}.

\section{Fluid Benchmark is Stronger}\label{appendix:fluid-relaxation}

\begin{proposition}
	For any collection of request distributions $\{\PP_t\}_t$, we have $$\E_{\{\gamma_t\}_t}[\opt(\{\gamma_t\}_t)] \leq \fluid(\{\PP_t\}_t)\,.$$
\end{proposition}
\begin{proof}
	Fix any request sequence $\{\gamma_t\}_t$ and let $\{x_t^*\}_t$ be an optimal solution to the corresponding hindsight optimization problem $\opt(\{\gamma_t\}_t)$. Then, $x_t(\gamma) = x_t^*$ for all $\gamma \in \S$ is a feasible solution of $\fluid(\{\PP_t\})$ and consequently, we have $\opt(\{\gamma_t\}_t) \leq \fluid(\{\PP_t\}_t)$. Since the request sequence $\{\gamma_t\}_t$ was arbitrary, we have $\E_{\{\gamma_t\}_t}[\opt(\{\gamma_t\}_t)] \leq \fluid(\{\PP_t\}_t)$, as required.
\end{proof}

\section{General Position}\label{appendix:general-position}

Given any collection of request distributions $\{\PP_t\}_t$ (which may or may not satisfy Assumption~\ref{assumption:general-position}), we can define perturbed distributions $\{\hat \PP_t\}_t$ to capture the distribution of perturbed requests $\hat \gamma_t = (\hat f_t, \hat b_t)$ generated using the following two step procedure: (i) Draw a request $\gamma_t = (f_t, b_t)$ according to the unperturbed distributions $\PP_t$; (ii) Add a perturbation by setting $\hat f_t(x) = f_t(x) + \epsilon_t \cdot x$ for all $x \in \X$, where $\epsilon_t \sim \text{Unif}([0,a])$, and leave the consumption function unchanged $\hat b_t( \cdot) = b_t(\cdot)$. Then, $\{\hat \PP_t\}_t$ satisfy Assumption~\ref{assumption:general-position} and $\left| \fluid(\{\PP_t\}_t) - \fluid(\{\hat \PP_t\}_t) \right| \leq a \cdot T$, where $a > 0$ can be made arbitrarily small.

\section{Efficiently Computing the Target Sequence}\label{appendix:fast-implmentation}

\begin{algorithm}[t!]
	\SetAlgoLined
	\textbf{Input:} Trace $\{\tilde \gamma_t\}_t$ in general position and sorted in increasing order of $\coeff(\tilde f_t)/\coeff(\tilde b_t)$.\\
	\textbf{Initialize:} Total payment $P= 0$ and target sequence $\lambda_t = 0$ for all $t \in [T]$.\\
	\For{$t = T, \ldots, 0$}{
        \If{$P + \tilde b_t(\bar x) > B$}{Set $\lambda_t \leftarrow \tilde b_t(\bar x)$, and set $\tilde \mu = \coeff(\tilde f_t)/\coeff(\tilde b_t)$. \textbf{Break}.}
        \Else{Update total payment $P \leftarrow P + \tilde b_t(\bar x)$ and set $\lambda_t \leftarrow \tilde b_t(\bar x)$.}
    }
	\Return{Dual variable $\tilde \mu$, target sequence $\{\lambda_t\}_t$}
	\caption{Learning the Dual from the Trace}
	\label{alg:efficient-trace-dual}
\end{algorithm}

In this section, we describe an efficient procedure for computing the empirical optimal dual solution $\tilde \mu$ and the target sequence $\{\lambda_t\}_t$. Consider a trace $\{\tilde \gamma_t\}_t$ and set $\tilde \gamma_0 = (f_0,b_0)$ such that $f_0(x) = b_0(x) = 0$ for all $x \in \X$. Without loss of generality, we will assume that $\{\tilde \gamma_t\}_t$ is sorted in increasing order of $\coeff(\tilde f_t)/\coeff(\tilde b_t)$ (assume $0/0 = 0$), i.e.,
\begin{align*}
    \frac{\coeff(\tilde f_s)}{\coeff( \tilde b_s)} \leq \frac{\coeff(\tilde f_t)}{\coeff( \tilde b_t)} \quad \forall\ s \leq t\,.
\end{align*}

This can be easily achieved by maintaining a sorted array with $O(\log(T))$ insertion time or sorting the array with $O(T \log(T))$ processing time. Moreover, since the trace $\{\tilde \gamma_t\}_t$ is in general position by Assumption~\ref{assumption:general-position}, all the $\coeff(\tilde f_t)/ \coeff(\tilde b_t)$ are distinct for $t \in [T]$.

\begin{theorem}
    $\tilde \mu$ returned by Algorithm~\ref{alg:efficient-trace-dual} is the smallest element in $$\argmin_{\mu \geq 0} \mu\cdot B + \sum_{t=1}^T \max_{x \in \X} \left\{ \tilde f_t(x) - \mu \cdot \tilde b_t(x) \right\}\,. $$ Moreover, $\lambda_t = b_t^*(\tilde \mu)$ for all $t \in [T]$.
\end{theorem}

\begin{proof}
Set $q(\mu) = \mu\cdot B + \sum_{t=1}^T \max_{x \in \X} \left\{ \tilde f_t(x) - \mu \cdot \tilde b_t(x) \right\}$. First, we show that the dual variable $\tilde \mu$ is smallest element in $\argmin_{\mu \geq 0} q(\mu)$. To do so, we consider the following two cases:

\begin{itemize}
    \item $\tilde \mu = 0$. In this case, the `If' condition implies that there exists $s \in [T]$ such that $\sum_{t=s}^T \tilde b_t(\bar x) \leq B$ and $\coeff(\tilde f_t) = 0$ for all $t< s$. Moreover, we have $0 \in \argmax_{x \in \X} \left\{ \tilde f_t(x) - \mu \cdot \tilde b_t(x) \right\}$ for all $t < s$ and $\bar x \in \argmax_{x \in \X} \left\{ \tilde f_t(x) - \mu \cdot \tilde b_t(x) \right\}$ for all $t \geq s$. Now, note that the set of sub-gradients of the maximum of a collection of linear functions is equal to convex hull of gradients of all the linear functions which are binding (for example, see Chapter 5 of \citealt{bertsekas2009convex}). Therefore, $B - \sum_{t=s}^T \tilde b(\bar x) \in \partial q(0)$. Since $B - \sum_{t=s}^T \tilde b(\bar x) \geq 0$, the definition of a subgradient implies that $q(0) \leq q(\mu)$ for all $\mu \geq 0$. Hence, we have shown that $\tilde \mu$ is the smallest minimizer of $q(\cdot)$, as required.

    \item $\tilde \mu > 0$. In this case, the `If' condition implies that there exists $s \in [T]$ such that $\sum_{t=s+1}^T \tilde b_t(\bar x) < B$, $\sum_{t=s}^T \tilde b_t(\bar x) > B$ and $\tilde f_s(x) - \tilde \mu \cdot \tilde b_s(x) = 0$ for all $x \in \X$. Moreover, we have $0 \in \argmax_{x \in \X} \left\{ \tilde f_t(x) - \mu \cdot \tilde b_t(x) \right\}$ for all $t < s$, $\bar x \in \argmax_{x \in \X} \left\{ \tilde f_t(x) - \mu \cdot \tilde b_t(x) \right\}$ for all $t > s$ and $\{0,\bar x\} \subseteq \argmax_{x \in \X} \left\{ \tilde f_s(x) - \mu \cdot \tilde b_s(x) \right\}$. Now, select a $\lambda \in [0,1]$ such that
    \begin{align*}
        B - \lambda \cdot \tilde b_s(0) + (1 - \lambda) \cdot \tilde b_s(\bar x) + \sum_{t=s+1}^T \tilde b_t(\bar x) = 0\,.
    \end{align*}
    Now, note that the set of sub-gradients of the maximum of a collection of linear functions is equal to convex hull of gradients of all the linear functions which are binding (for example, see Chapter 5 of \citealt{bertsekas2009convex}). Therefore, $0 =B - \lambda \cdot \tilde b_s(0) + (1 - \lambda) \cdot \tilde b_s(\bar x) + \sum_{t=s+1}^T \tilde b_t(\bar x) \in \partial q(\tilde\mu)$. Consequently, the definition of a subgradient implies that $q(0) \leq q(\mu)$ for all $\mu \geq 0$. Finally, consider any $\mu < \tilde \mu$. Then, $\tilde f_s(x) - \mu \cdot \tilde b_s(x) > 0$, which further implies $\{\bar x\} = \argmax_{x \in \X} \left\{ \tilde f_t(x) - \mu \cdot \tilde b_t(x) \right\}$ for all $t \geq s$. Therefore, for any $\{x_t\}_t$ such that $x_t \in \argmax_{x \in \X} \left\{ \tilde f_t(x) - \mu \cdot \tilde b_t(x) \right\}$, we have $B - \sum_{t=1}^T \tilde b_t(x_t) > 0$. Hence, $v > 0$ for all $v \in \partial q(\mu)$ and consequently $\mu$ is a minimizer of $q(\cdot)$. Hence, we have shown that $\tilde \mu$ is the smallest minimizer of $q(\cdot)$, as required.
\end{itemize}

Finally, we show that $\lambda_t = \tilde b_t^*(\tilde \mu)$ for all $t \in [T]$. Let $s$ be the value of $t$ at which the `For' loop terminates. From the definition of $\tilde \mu$, we have $\{\bar x\} = \argmax_{x \in \X} \left\{ \tilde f_t(x) - \tilde \mu \cdot \tilde b_t(x) \right\}$ for all $t > s$, $\X = \argmax_{x \in \X} \left\{ \tilde f_s(x) - \tilde \mu \cdot \tilde b_s(x) \right\}$ and $\{ 0 \} = \argmax_{x \in \X} \left\{ \tilde f_t(x) - \tilde \mu \cdot \tilde b_t(x) \right\}$ for all $t < s$. Therefore, $\tilde b_t^*(\tilde \mu) = \tilde b_t(\bar x) = \lambda_t$ for all $t \geq s$ and $\tilde b_t^*(\tilde \mu) = \tilde b_t(0) = \lambda_t$ for all $t < s$.
\end{proof}
\section{Missing Proofs from Section~\ref{sec:dual}}

\subsection{Proof of Lemma~\ref{lemma:trace-dual}}

\begin{proof}[Proof of Lemma~\ref{lemma:trace-dual}]
	Define $q(\mu) = \mu\cdot B + \sum_{t=1}^T \max_{x \in \X} \left\{\tilde f_t(x) - \mu \cdot \tilde b_t(x) \right\}$. Then, $q(\cdot)$ is a convex function of $\mu$ because the maximum of a collection of linear function is convex and the sum of convex function is also convex \citep{bertsekas2009convex}. Since $\tilde \mu \in \argmin_{\mu \geq 0} q(\mu)$, first-order condition of optimality (Proposition 5.4.7 of \citealt{bertsekas2009convex}) implies that one of the following statements holds:
	\begin{itemize}
		\item[(i)] $\tilde \mu = 0$ and there exists $v \in \partial q(0)$ such that $v \geq 0$. 
		\item[(ii)] Zero is a sub-differential of $q$ at $\tilde \mu$, i.e.,  $0 \in \partial q(\tilde \mu)$.
	\end{itemize}
	
	Recall that the trace $\{\tilde \gamma_t\}_t$ is assumed to be in general position with probability one. Therefore, there is at most one time step for which $\argmax_{x \in \X} f_t(x) - \tilde \mu \cdot b_t(x)$ is not unique. Let $s$ be that time step. Now, note that the set of sub-gradients of the maximum of a collection of linear functions is equal to convex hull of gradients of all the linear functions which are binding (for example, see Chapter 5 of \citealt{bertsekas2009convex}). Hence, $v \in \partial q(\tilde\mu)$ implies the existence of $\D_s \in \Delta(\X)$ such that
	\begin{align*}
		\text{Support}(\D_s) \subseteq \argmax_{x \in \X} \tilde f_s(x) - \tilde\mu \cdot \tilde b_s(x) \quad \text{ and } \quad v = B - \E_{x \sim \D_s}[\tilde b_s(x)] - \sum_{t\neq s}^T \tilde b_t^*(\tilde \mu) \,.
	\end{align*}
	where $x_t^*(\gamma_t, \tilde \mu)$ is the optimal solution to $\max_{x \in \X} \tilde f_t(x) - \tilde \mu \cdot \tilde b_t(x)$ as described in Definition~\ref{definition:profit-maximizing-decision}. Since $0 \leq b_s(x) \leq \bar b$ for all $x \in \X$, we get
	\begin{align*}
		\left| B - v - \sum_{t=1}^T \tilde b_t^*(\tilde \mu) \right| \leq \bar b\,,
	\end{align*}
	where we have used $0 \leq b_t(x) \leq \bar b$ for all $x \in \X$ and $t \in [T]$. Therefore, statements (i) and (ii) imply that either $\tilde \mu = 0$ and $\bar b +  B - \sum_{t=1}^T \tilde b_t^*(\tilde \mu) \geq v \geq 0$, or  $\left|B - \sum_{t=1}^T \tilde b_t^*(\tilde \mu) \right| \leq \bar b$, as required.
\end{proof}

\subsection{Proof of Theorem~\ref{thm:concentration}}

\begin{proof}[Proof of Theorem~\ref{thm:concentration}]

    Define the hypothesis class
    \begin{align*}
        \mathcal{F}\coloneqq \{ (f,b) \mapsto b^*(\mu) \mid \mu \geq 0 \} \,.
    \end{align*}

    Let $\text{Rad}(\cdot)$ denote Radmacher complexity. Then, we know from PAC learning theory (for example see Chapter 26 of \citealt{shalev2014understanding}) that
    \begin{align}\label{eq:radmacher1}
		\Pr_{\{\tilde \gamma\}_t \sim \prod_t \tilde \PP_t}\left( \sup_{\mu \geq 0} \left| \sum_{t=1}^T \tilde b_t^*(\mu) - \sum_{t=1}^T \E_{\hat \gamma_t \sim \tilde \PP_t}\left[\hat b_t^*(\mu) \right] \right| \geq r(T) \right) \leq  \frac{1}{T^2} \,. 
	\end{align}
    for
    \begin{align*}
        r(T) \geq 2T \cdot \E_{\{\hat \gamma_t\}_t \sim \prod_t \tilde \PP_t} \left[\text{Rad}(\mathcal{F}\ \circ\ \{\hat \gamma_t\}_t) \right] + 2\bar b \cdot \sqrt{T\log(2T)}\,.
    \end{align*}

    Let $H(\{\hat \gamma_t\}_t) = \left\{ \{\hat b_t^*(\mu) \}_t \mid \mu \geq 0 \right\}$ denote the set of all possible resource expenditure vectors that can be generated from a trace $\{\hat \gamma_t\}_t$, then
    \begin{align}\label{eq:radmacher2}
        \E_{\{\hat \gamma_t\}_t \sim \prod_t \tilde \PP_t} \left[\text{Rad}(\mathcal{F}\ \circ\ \{\hat \gamma_t\}_t) \right] = \E_{\{\hat \gamma_t\}_t \sim \prod_t \tilde \PP_t} \left[\text{Rad}(H(\{\hat\gamma_t\}_t))\right] = \frac{1}{T} \cdot \E_{\{\hat \gamma_t\}_t \sim \prod_t \tilde \PP_t}\E_{\vec\sigma} \left[  \sup_{\mu \geq 0}   \sum_{t=1}^T \sigma_t \cdot \hat b_t^*(\mu) \right]\,,
    \end{align}
    where $\{\sigma_t\}_t$ are independent Radmacher random variables.
    
    For a linear function $f:\R \to \R$, let $\coeff(f)$ denote its coefficient. Moreover, let $\bar x = \max_{x \in \X} x$. Then, observe that for a request $\gamma = (f,b)$ and dual variable $\mu \geq 0$, we have
    \begin{align*}
        x^*(\gamma, \mu) = 
        \begin{cases}
            \bar x &\text{if } \coeff(f) - \mu \cdot \coeff(b) \geq 0 \text{ and } \coeff(f) \neq 0\\
            0 &\text{otherwise} 
        \end{cases}\,.
    \end{align*}
    Therefore, for any request $\gamma = (f,b)$ with $\coeff(f) \neq 0$ and $\coeff(b) \neq 0$, there exists a critical $\mu^* = \coeff(f)/\coeff(b)$ such that
    \begin{align*}
        b^*(\mu) =
        \begin{cases}
            b(\bar x) &\text{if } \mu \leq \mu^*\\
            0 &\text{if } \mu > \mu^*
        \end{cases}\,.
    \end{align*}
    For $\gamma = (f,b)$ with $\coeff(f) = 0$ or $\coeff(b) \neq 0$, we have $b^*(\mu) = 0$ for all $\mu \geq 0$. Consider the trace $\{\hat \gamma_t\}_t$, and let $\mu^*_t$ be the critical dual solution for request $\hat \gamma_t$ as defined above. Then, the assumption that the trace is in general position (Assumption~\ref{assumption:general-position}) implies that the critical points $\{\mu^*_t\}_t$ are distinct. Consequently, we get that $\{\hat b_t^*(\mu)\}_t$ remains constant whenever $\mu$ lies between any two critical points. Since the total number of critical points is $T$, we get that $|H(\{\hat \gamma_t\}_t)| \leq T$. Therefore, Massart Lemma applies and we get
    \begin{align*}
        \text{Rad}(H(\{\hat\gamma_t\}_t)) \leq \bar b \cdot \sqrt{\frac{2 \log(T)}{T}}\,.
    \end{align*}
    Combining this with \eqref{eq:radmacher1} and \eqref{eq:radmacher2} yields the theorem.
\end{proof}

\subsection{Proof of Theorem~\ref{thm:learned-dual-regret}}

\begin{proof}[Proof of Theorem~\ref{thm:learned-dual-regret}]

	By Assumption~\ref{assumption:general-position}, the request sequence $\{\gamma_t\}$ is in general position almost surely. Therefore, there is at most 1 time step such that $f_t(x_t') \neq f_t^*(\tilde\mu)$, call it $s$. Let $\tA$ be the first time step $t$ in which $B_{t+1} \leq \bar b$, i.e., $\sum_{t=1}^{\tA} b_t^*(\tilde \mu) \geq B - \bar b$. Then,
	\begin{align*}
		\E\left[ R(A|\{\gamma_t\}_t) \right]  &= \E\left[ \sum_{t=1}^{\tA} f_t(x_t') \right]\\
        &\geq \E\left[ \sum_{t=1}^{\tA} f_t^*(\tilde \mu) \right] - |f_s^*(\tilde\mu) - f_s(x_s')| \\
		&\geq \E\left[ \sum_{t=1}^T f_t^*(\tilde \mu) \right] -  \E\left[\sum_{t=\tA +1}^T f_t^*(\tilde\mu) \right] - \bar f \\
		&\geq \E\left[ \sum_{t=1}^T f_t^*(\tilde \mu) \right] -  \E\left[\kappa \cdot \sum_{t=\tA +1}^T b_t^*(\tilde\mu) \right] - \bar f \\
		&\geq D(\tilde \mu| \{\PP_t\}_t) - \E\left[ \tilde \mu \cdot \left(B - \sum_{t=1}^T b_t^*(\tilde \mu) \right) \right] -  \E\left[\kappa \cdot \sum_{t=\tA +1}^T b_t^*(\tilde \mu) \right] - \bar f\\
		&\geq \fluid(\{\PP_t\}_t) - \E\left[ \tilde \mu \cdot \left(B - \sum_{t=1}^T b_t^*(\tilde \mu) \right) \right] -  \E\left[\kappa \cdot \sum_{t=\tA +1}^T b_t^*(\tilde \mu) \right] - \bar f \,.
	\end{align*}
	Therefore,
	\begin{align*}
		\regret(A) \leq  \E\left[ \kappa \cdot \sum_{t=\tA +1}^T b_t^*(\tilde \mu) \right] + \E\left[ \tilde \mu \cdot \left(B - \sum_{t=1}^T b_t^*(\tilde \mu) \right) \right] + \bar f\,.
	\end{align*}
	
	In the remainder of the proof, we bound the first two terms on the RHS. 
	
	For the first term, observe that
	\begin{align}\label{eq:excess-spend}
		\sum_{t =\tA + 1}^T b_t^*(\tilde \mu) &\leq \left( \sum_{t=1}^T b_t^*(\tilde\mu) \right) - (B - \bar b) \nonumber \\
		&=  \left( \sum_{t=1}^T b_t^*(\tilde\mu) - \sum_{t=1}^T\tilde b_t^*(\tilde\mu) \right) - \left(B - \sum_{t=1}^T \tilde b_t^*(\tilde\mu) \right) + \bar b \nonumber \\
		&\leq \left| \sum_{t=1}^T b_t^*(\tilde\mu) - \sum_{t=1}^T\tilde b_t^*(\tilde\mu)  \right| + 2 \cdot \bar b \,,
	\end{align}
	where the first inequality follows from the definition of $\tA$ and the last inequality follows from Lemma~\ref{lemma:trace-dual}.

	For the second term, observe that Lemma~\ref{lemma:trace-dual} implies
	\begin{align*}
		\tilde \mu \cdot \left(B - \sum_{t=1}^T b_t^*(\tilde \mu) \right) &= \tilde \mu \cdot \left(B - \sum_{t=1}^T \tilde b_t^*(\tilde \mu) \right) + \tilde \mu \cdot \left( \sum_{t=1}^T \tilde b_t^*(\tilde \mu) - \sum_{t=1}^T b_t^*(\tilde \mu) \right) \\
		&\leq \tilde \mu  \cdot \bar b + \tilde \mu \cdot \left| \sum_{t=1}^T \tilde b_t^*(\tilde \mu) - \sum_{t=1}^T b_t^*(\tilde \mu) \right| \,.
	\end{align*}
	
	Note that $\tilde \mu \leq \kappa$. This is because the definition of $\kappa$ implies that $\max_{x \in \X} f(x) - \mu \cdot b(x) = 0$ for all $\gamma = (f,b) \in \S$ and $\mu \geq \kappa$. Hence,
	\begin{align*}
		\mu\cdot B + \sum_{t=1}^T \max_{x \in \X} \left\{\tilde f_t(x) - \mu \cdot \tilde b_t(x) \right\} = \mu \cdot B < \kappa \cdot B= \kappa \cdot B + \sum_{t=1}^T \max_{x \in \X} \left\{\tilde f_t(x) - \kappa \cdot \tilde b_t(x) \right\}
	\end{align*}
	for all $\mu > \kappa$. Therefore, we get
	\begin{align}\label{eq:slackness-term}
		\tilde \mu \cdot \left(B - \sum_{t=1}^T b_t^*(\tilde \mu) \right) \leq \kappa \cdot \bar b + \kappa \cdot \left| \sum_{t=1}^T \tilde b_t^*(\tilde \mu) - \sum_{t=1}^T b_t^*(\tilde \mu) \right| \,.
	\end{align}

	Define $G$ to be the good event to be one in which the total expenditures under the trace and the requests sequence are close:
	\begin{align*}
		\sup_{\mu \geq 0} \left| \sum_{t=1}^T \tilde b_t^*(\mu) -  \sum_{t=1}^T b_t^*(\mu) \right| \leq r(T) \,.
	\end{align*}
	
	Then, Theorem~\ref{thm:concentration} and Union Bound imply that $\Pr(G^c) \leq 2/T^2$ and $\Pr(G) \geq 1/2/T^2$. Finally, we can put it all together to get the required bound:
	\begin{align*}
		\regret(A) &\leq  \E\left[ \kappa \cdot \sum_{t=\tA +1}^T b_t^*(\tilde \mu) \right] + \E\left[ \tilde \mu \cdot \left(B - \sum_{t=1}^T b_t^*(\tilde \mu) \right) \right] + \bar f\\
		&\leq \E\left[ \kappa \cdot \left| \sum_{t=1}^T \tilde b_t^*(\tilde \mu) - \sum_{t=1}^T b_t^*(\tilde \mu) \right| \right] + 2 \kappa \bar b + \E\left[ \kappa \cdot \left| \sum_{t=1}^T \tilde b_t^*(\tilde \mu) - \sum_{t=1}^T b_t^*(\tilde \mu) \right| \right] + \kappa \bar b + \kappa \bar b \\
		&= 2 \kappa \cdot \E\left[\left| \sum_{t=1}^T \tilde b_t^*(\tilde \mu) - \sum_{t=1}^T b_t^*(\tilde \mu) \right|\ \biggr|\ G \right] \Pr(G) + 2 \kappa \cdot \E\left[\left| \sum_{t=1}^T \tilde b_t^*(\tilde \mu) - \sum_{t=1}^T b_t^*(\tilde \mu) \right|\ \biggr|\ G^c \right] \Pr(G^c) + 4 \kappa \bar b\\
		&\leq 2 \kappa \cdot r(T) + 2 \kappa \cdot 2 T \bar b \cdot \frac{2}{T^2} + 4\kappa \bar b\\
		&\leq 12 \kappa \bar b + 2 \kappa r(T)\,. \qedhere
	\end{align*}
\end{proof}

\section{Missing Proofs from Section~\ref{sec:descent}}\label{appendix:descent}

\subsection{Proof of Theorem~\ref{thm:dual-descent-regret}}

\begin{proof}[Proof of Theorem~\ref{thm:dual-descent-regret}]

	Let $\tA$ be the first time less than $T$ for which $\sum_{t=1}^{\tA} b_{t}(x_t) + \bar{b} \ge B$. Set $\tA = T$ if this inequality is never satisfied. Then, $x_t = x_t'$ for all $t \leq \tA$ and $\sum_{t=1}^{\tA} b_t(x_t') \geq B - \bar b$.
	
	First, observe that
	\begin{align}\label{eq:alg-reward}
		R(A| \{\gamma_t\}_t) \geq \sum_{t=1}^{\tA} f_t(x_t') = \sum_{t=1}^T f_t(x_t') - \sum_{t=\tA +1}^T f_t(x_t') \geq \sum_{t=1}^Tf_t(x_t') - \kappa \cdot \sum_{t=\tA +1 }^T b_t(x_t')\,.
	\end{align}

	Next observe that, for all $t \in [T]$, $\mu_t$ is independent of $\gamma_t$ because $\mu_t$ is completely determined by $\{\gamma_1, \dots, \gamma_{t-1}\}$. Hence,
	\begin{align*}
		\E_{\gamma_t}\left[ f_t(x_t') \mid \mu_t \right] &= \E_{\gamma_t}\left[ f_t(x_t') + \mu_t \cdot (\beta_t - b_t(x_t')) \mid \mu_t \right] - \E_{\gamma_t}\left[ \mu_t \cdot (\lambda_t - b_t(x_t')) \mid \mu_t \right] - \E_{\gamma_t}\left[ \mu_t \cdot (\beta_t - \lambda_t) \mid \mu_t \right]\\
		&= \E_{\gamma_t}\left[ D(\mu_t|\PP_t, \beta_t) | \mu_t \right] - \E_{\gamma_t}\left[ \mu_t \cdot (\lambda_t - b_t(x_t')) \mid \mu_t \right] - \E_{\gamma_t}\left[ \mu_t \cdot (\beta_t - \lambda_t) \mid \mu_t \right] \,.
	\end{align*}
	Taking unconditional expectations on both sides and applying the tower rule yields
	\begin{align*}
		\E\left[ f_t(x_t')\right] = \E\left[ D(\mu_t|\PP_t, \beta_t) \right] - \E\left[ \mu_t \cdot (\lambda_t - b_t(x_t'))\right] - \E\left[ \mu_t \cdot (\beta_t - \lambda_t)\right] \,.
	\end{align*}
	Summing over $t \in [T]$, we get
	\begin{align}\label{eq:tot-reward}
		\sum_{t=1}^T \E[f_t(x_t')] = \sum_{t=1}^T \E\left[ D(\mu_t|\PP_t, \beta_t) \right] - \sum_{t=1}^T \E\left[ \mu_t \cdot (\lambda_t - b_t(x_t'))\right] - \sum_{t=1}^T \E\left[ \mu_t \cdot (\beta_t - \lambda_t)\right]\,.
	\end{align}
	
	Therefore, \eqref{eq:alg-reward} and \eqref{eq:tot-reward} together imply
	\begin{align}\label{eq:interim-regret}
		\E\left[ \left\{ \sum_{t=1}^T D(\mu_t|\PP_t, \beta_t) \right\} - R(A|\{\gamma_t\}_t) \right] \leq \E\left[ \sum_{t=1}^T \mu_t \cdot (\lambda_t - b_t(x_t'))  + \kappa \cdot \sum_{t=\tA + 1}^T b_t(x_t') \right] + \sum_{t=1}^T \E[\mu_t \cdot (\beta_t - \lambda_t)]\,.
	\end{align}

	\textbf{FTRL Regret Bound.} Define $w_t(\mu) \coloneqq \mu \cdot (\lambda_t - b_t(x_t'))$. Then, Algorithm~\ref{alg:dual-descent} can be seen as running FTRL with linear losses $\{w_t(\cdot)\}_t$. The gradients of these loss functions are given by $\nabla w_t(\mu) = \lambda_t - b_t(x_t')$, which satisfy $\|\nabla w_t(\mu)\|_\infty \leq \|b_t(x_t')\|_\infty + \|\lambda_t\|_\infty \leq \bar b + \bar \lambda$. Therefore, the regret bound for FTRL implies that for all $\mu \geq 0$:
	\begin{align}
		\sum_{t=1}^{T} w_t(\mu_t) - w_t(\mu) \le E(T,\mu)\,,
	\end{align}
	where $E(T,\mu) = \frac{2(\bar{b} + \bar\lambda)^2}{\sigma} \eta \cdot T + \frac{h(\mu) - h(\mu_1)}{\eta}$ is the regret bound of FTRL after $T$ iterations \citep{hazan2016introduction}. Equivalently, we can write
	\begin{align*}
		\sum_{t=1}^T \mu_t \cdot (\lambda_t - b_t(x_t')) \leq E(T,\mu) + \sum_{t=1}^T \mu \cdot (\lambda_t - b_t(x_t')) \qquad \forall\ \mu \geq 0.
	\end{align*}
	
	Now, consider the following two cases:
	\begin{itemize}
		\item Case 1: $\tA = T$. Here, setting $\mu = 0$ yields
			\begin{align*}
				\sum_{t=1}^T \mu_t \cdot (\lambda_t - b_t(x_t'))  + \kappa \cdot \sum_{t=\tA + 1}^T b_t(x_t') \leq E(T,0) \,.
			\end{align*}
		\item Case 2: $\tA < T$. Then, $\sum_{t=1}^{\tA} b_t(x_t') \geq B - \bar b$. Hence, setting $\mu = \kappa$ yields
			\begin{align*}
				\sum_{t=1}^T \mu_t \cdot (\lambda_t - b_t(x_t'))  + \kappa \cdot \sum_{t=\tA + 1}^T b_t(x_t') &\leq E(T, \kappa) + \sum_{t=1}^T \kappa \cdot (\lambda_t - b_t(x_t'))  + \kappa \cdot \sum_{t=\tA + 1}^T b_t(x_t')\\
				&= E(T, \kappa) + \kappa \cdot \left(\sum_{t=1}^T \lambda_t - \sum_{t=1}^{\tA} b_t(x_t') \right)\\
				&\leq E(T, \kappa) + \kappa \cdot \left( \left\{\sum_{t=1}^T \lambda_t \right\} - (B - \bar b) \right)\\
				&= E(T, \kappa) + \kappa\bar b + \kappa \cdot \left(\left\{\sum_{t=1}^T \lambda_t \right\} - B \right)\,.
			\end{align*}
	\end{itemize}
	
	Combining the two cases implies that for all values of $\tA$ we have
	\begin{align}\label{eq:comp-slack}
		\sum_{t=1}^T \mu_t \cdot (\lambda_t - b_t(x_t'))  + \kappa \cdot \sum_{t=\tA + 1}^T b_t(x_t') \leq \max\{E(T,0), E(T,\kappa)\} + \kappa \bar b + \kappa \cdot \left(\left\{\sum_{t=1}^T \lambda_t \right\} - B \right)^+ \,.
	\end{align}
	
	Finally, combining \eqref{eq:interim-regret} and \eqref{eq:comp-slack} yields
	\begin{align*}
		\E\left[ \left\{ \sum_{t=1}^T D(\mu_t|\PP_t, \beta_t) \right\} - R(A|\{\gamma_t\}_t) \right] \leq &\max\{E(T,0), E(T,\kappa)\} + \kappa \bar b + \kappa \cdot \left(\left\{\sum_{t=1}^T \lambda_t \right\} - B \right)^+\\
		 &+ \sum_{t=1}^T \E[\mu_t \cdot (\beta_t - \lambda_t)]\,.
	\end{align*}
	Plugging in the definition of $E(T,\mu)$ finishes the proof.
\end{proof}

\subsection{Proof of Lemma~\ref{lemma:discrepancy-term}}

\begin{proof}[Proof of Lemma~\ref{lemma:discrepancy-term}]
	Consider any two request distributions $\PP, \tilde \PP \in \Delta(\S)$. Then, by the definition of the Wasserstein metric, there exists a joint probability distribution $Q$, with marginals $\PP$ and $\tilde \PP$ on the first and second factors respectively, such that
	\begin{align*}
		\W(\PP, \tilde\PP) = \E_{(\gamma, \tilde \gamma) \sim Q} \left[ \sup_x |f(x) - \tilde f(x)| + \sup_x |b(x) - \tilde b(x)| \right] \,.
	\end{align*}
	Let $x^*(\gamma, \mu)$ be the optimal solution of $\max_{x \in X} f(x) - \mu \cdot b(x)$ for request $\gamma = (f,b)$ as described in Definition~\ref{definition:profit-maximizing-decision}. Then, for any $\mu \in [0,\kappa]$ and $x: \S \to \X$, we have
	\begin{align}\label{eq:wasserstein}
		&\E_{(\gamma, \tilde \gamma) \sim Q} \left[ \left| f(x(\gamma)) - \mu \cdot b(x(\gamma)) - \left\{ \tilde f(x(\gamma)) -  \mu \cdot \tilde b(x(\gamma)) \right\} \right| \right] \nonumber \\
		\leq & \E_{(\gamma, \tilde \gamma) \sim Q} \left[ \left| f(x(\gamma)) - \tilde f(x(\gamma))\right| + \mu \cdot \left| b(x(\gamma)) - \tilde b(x(\gamma)) \right| \right] \nonumber \\
		\leq & \W(\PP, \tilde \PP) + \kappa \cdot \W(\PP, \tilde\PP) \nonumber\\
		=& (1 + \kappa) \cdot \W(\PP, \tilde \PP)\,.
	\end{align}
	
	Now, for $t \in [T]$, we can use \eqref{eq:wasserstein} to write
	\begin{align*}
		& D(\mu_t|\tilde \PP, \beta_t) - D(\mu_t| \PP, \beta_t)  \\
		=& \E_{\tilde\gamma \sim \tilde \PP}[ \tilde f(x^*(\tilde\gamma, \mu_t)) - \mu_t \cdot \tilde b(x^*(\tilde \gamma, \mu_t)) + \mu_t\cdot \beta_t ] - \E_{\gamma \sim \PP}[ f(x^*(\gamma, \mu_t)) - \mu_t \cdot b(x^*(\gamma, \mu_t)) + \mu_t \cdot \beta_t ] \\
		\leq&  \E_{\tilde\gamma \sim \tilde \PP, \gamma \sim \PP}[ \tilde f(x^*(\tilde \gamma, \mu_t)) - \mu_t \cdot \tilde b(x^*( \tilde \gamma, \mu_t)) + \mu_t\cdot \beta_t - \left\{ f(x^*(\tilde \gamma, \mu_t)) - \mu_t \cdot b(x^*(\tilde \gamma, \mu_t)) + \mu_t \cdot \beta_t \right\} ] \\
		\leq& \E_{(\gamma, \tilde \gamma) \sim Q} \left[ \left| f(x^*(\tilde \gamma, \mu_t)) - \mu_t \cdot b(x^*(\tilde \gamma, \mu_t)) - \left\{ \tilde f(x^*(\tilde \gamma, \mu_t)) -  \mu_t \cdot \tilde b(x^*(\tilde \gamma, \mu_t)) \right\} \right| \right]\\
		\leq& (1 + \kappa) \cdot \W(\PP, \tilde\PP)\,,
	\end{align*}
	where the first inequality follows from the definition of $x^*(\gamma, \mu_t)$ and the second inequality follows from the fact that $(\PP, \tilde\PP)$ are the marginals of $Q$. As a consequence, we get
	\begin{align}\label{eq:wasserstein-inter}
		\sum_{t=1}^T D(\mu_t|\PP_t, \beta_t) &= \sum_{t=1}^T D(\mu_t | \tilde \PP_t, \beta_t) - \sum_{t=1}^T \left\{ D(\mu_t|\tilde \PP_t, \beta_t) - D(\mu_t| \PP_t, \beta_t) \right\} \nonumber\\
		&\geq \sum_{t=1}^T D(\mu_t | \tilde \PP_t, \beta_t) - (1 + \kappa) \cdot \sum_{t=1}^T \W(\PP_t, \tilde\PP_t) \nonumber\\
		&\geq \sum_{t=1}^T \fluid(\tilde \PP_t, \beta_t) - (1 + \kappa) \cdot \sum_{t=1}^T \W(\PP_t, \tilde\PP_t)\,, 
	\end{align}
	where the second inequality follows from weak duality.
	
	Next, observe that the definition of $\beta_t = \E_{\hat \gamma \sim \tilde \PP_t}[ \hat b^*(\tilde \mu)]$ implies that $x^*(\tilde \gamma_t, \tilde \mu)$ is a feasible to solution to the optimization problem which defines $\fluid(\tilde \PP_t, \beta_t)$. Hence,
	\begin{align*}
		\sum_{t=1}^T \fluid(\tilde \PP_t, \beta_t) &\geq \sum_{t=1}^T \E_{\tilde \gamma_t \sim \tilde \PP_t} \left[ \tilde f_t(x^*(\tilde \gamma_t, \tilde \mu)) \right]\\
		&= \sum_{t=1}^T \E_{\tilde \gamma_t \sim \tilde \PP_t} \left[ \tilde f_t(x^*(\tilde \gamma_t, \tilde \mu)) - \tilde\mu \cdot \tilde b_t(x^*(\tilde \gamma_t, \tilde \mu))\right] + \tilde\mu \cdot \sum_{t=1}^T \E_{\tilde \gamma_t \sim \tilde \PP_t} \left[ \tilde b_t(x^*(\tilde \gamma_t, \tilde \mu)) \right]\,.
	\end{align*}
	Let $\{x_t(\cdot)\}_t$ be an optimal solution for $\fluid(\{\PP_t\})$. Then, for all $t \in [T]$, we have
	\begin{align*}
		\E_{\tilde \gamma_t \sim \tilde \PP_t} \left[ \tilde f_t(x^*(\tilde \gamma_t, \tilde \mu)) - \tilde\mu \cdot \tilde b_t(x^*(\tilde \gamma_t, \tilde \mu))\right] &= \E_{(\gamma_t, \tilde\gamma_t) \sim Q} \left[ \tilde f_t(x^*(\tilde \gamma_t, \tilde \mu)) - \tilde\mu \cdot \tilde b_t(x^*(\tilde \gamma_t, \tilde \mu))\right]\\
		&\geq \E_{(\gamma_t, \tilde\gamma_t) \sim Q} \left[ \tilde f_t(x_t(\gamma_t)) - \tilde\mu \cdot \tilde b_t(x(\gamma_t))\right] \\
		&\geq \E_{(\gamma_t, \tilde\gamma_t) \sim Q} \left[ f_t(x_t(\gamma_t)) - \tilde\mu \cdot b_t(x(\gamma_t))\right] - (1 + \kappa) \cdot \W(\PP_t, \tilde \PP_t)\\
		&= \E_{\gamma_t \sim \PP_t} \left[ f_t(x_t(\gamma_t)) - \tilde\mu \cdot b_t(x(\gamma_t))\right] - (1 + \kappa) \cdot \W(\PP_t, \tilde \PP_t)\,,
	\end{align*}
	where the first inequality follows from the definition of $x^*(\tilde \gamma_t, \tilde\mu)$ and the second inequality follows from \eqref{eq:wasserstein}. Therefore,
	\begin{align*}
		&\sum_{t=1}^T \fluid(\tilde \PP_t, \beta_t)\\
		 \geq & \sum_{t=1}^T \E_{\gamma_t} \left[ f_t(x_t(\gamma_t)) - \tilde\mu \cdot b_t(x(\gamma_t))\right] - (1 + \kappa) \cdot \sum_{t=1}^T \W(\PP_t, \tilde \PP_t) + \tilde\mu \cdot \sum_{t=1}^T \E_{\tilde \gamma_t} \left[ \tilde b_t(x^*(\tilde \gamma_t, \tilde \mu)) \right]\\
		=& \sum_{t=1}^t \E_{\gamma_t \sim \PP_t}[f_t(x(\gamma_t))] - \tilde \mu\cdot \left(\sum_{t=1}^T \E_{\gamma_t}[b_t(x(\gamma_t))] -\sum_{t=1}^T \E_{\tilde \gamma_t}[\tilde b_t^*(\tilde \mu)] \right) - (1 + \kappa) \cdot \sum_{t=1}^T \W(\PP_t, \tilde \PP_t)\\
		\geq& \fluid(\{\PP_t\}_t) - \tilde \mu \cdot  \left(B - \sum_{t=1}^T \beta_t \right) - (1 + \kappa) \cdot \sum_{t=1}^T \W(\PP_t, \tilde \PP_t)\,,
	\end{align*}
	where the second inequality follows from the feasibility of the optimal solution $\{x_t(\cdot)\}_t$, i.e.,\\ ${\sum_{t=1}^T \E_{\gamma_t}[b_t(x(\gamma_t))] \leq B}$. Combining this with \eqref{eq:wasserstein-inter} yields
	\begin{align*}
		\sum_{t=1}^T D(\mu_t| \PP_t, \beta_t) \geq \fluid(\{\PP_t\}_t) - \tilde \mu \cdot \left(B - \sum_{t=1}^T \beta_t \right) - 2(1 + \kappa) \cdot \sum_{t=1}^T \W(\PP_t, \tilde \PP_t)\,,
	\end{align*}
	as required.
\end{proof}

\subsection{Proof of Lemma~\ref{lemma:monotonicity}}

\begin{proof}[Proof of Lemma~\ref{lemma:monotonicity}]
	The definition of $x$ and $x'$ implies
	\begin{align*}
		f(x) - \mu \cdot b(x) \geq f(x') - \mu \cdot b(x') \quad \text{ and } \quad f(x') - \mu' \cdot b(x') \geq f(x) - \mu' \cdot b(x)\,.
	\end{align*} 
	Combining the two inequalities, we get
	\begin{align*}
		&f(x) - \mu \cdot b(x) - \left\{ f(x) - \mu' \cdot b(x) \right\} \geq f(x') - \mu \cdot b(x') - \left\{ f(x') - \mu' \cdot b(x') \right\} \\
		\implies & (\mu - \mu') \cdot (b(x') - b(x)) \geq 0\,.
	\end{align*}
	The lemma follows from the last inequality because $\mu - \mu' > 0$.
\end{proof}

\subsection{Proof of Lemma~\ref{lemma:change-in-target}}

\begin{proof}[Proof of Lemma~\ref{lemma:change-in-target}]

	Define
	\begin{align*}
		q(\mu) \coloneqq \mu\cdot B + \sum_{t=1}^T \max_{x \in \X} \left\{ \tilde f_t(x) - \mu \cdot \tilde b_t(x) \right\} \quad \text{and} \quad q^{(-s)}(\mu) \coloneqq \mu\cdot B + \sum_{t\neq s} \max_{x \in \X} \left\{ \tilde f_t(x) - \mu \cdot \tilde b_t(x) \right\} \,.
	\end{align*}
	
	We start by proving $\tilde \mu \geq \tilde \mu^{(-s)}$. For contradiction, suppose $\tilde \mu < \tilde \mu^{(-s)}$. Consider the following two cases:
    \begin{itemize}
        \item Case I: $0 \in \argmax_{x \in \X} \tilde f_s(x) - \tilde \mu \cdot \tilde b_s(x)$. Then, we must have $0 \in \argmax_{x \in \X} \tilde f_s(x) - \mu \cdot \tilde b_s(x)$ for all $\mu \geq \tilde \mu$. This is because, for $\mu \geq \tilde \mu$, Lemma~\ref{lemma:monotonicity} implies that $\tilde b_s(x') \leq \tilde b_s(0) = 0$ for all $x' \in \argmax_{x \in \X} \tilde f_s(x) - \mu \cdot \tilde b_s(x)$, and $\tilde f_s(x) \leq \kappa \cdot \tilde b_s(x)$ for all $x \in \X$.  Therefore, $q(\mu) = q^{(-s)}(\mu)$ for all $\mu \geq \tilde \mu$. Since $\tilde \mu$ is a minimizer of $q(\cdot)$ and $\tilde \mu^{(-s)} > \tilde\mu$, we get that
        \begin{align*}
            q^{(-s)}(\tilde \mu) = q(\tilde\mu) \leq q(\tilde\mu^{(-s)}) = q^{(-s)}\left(\tilde \mu^{(-s)} \right)\,.
        \end{align*}
        On the other hand, $\tilde \mu^{(-s)}$ is a minimizer of $q^{(-s)}(\cdot)$, which implies $q^{(-s)}\left(\tilde \mu^{(-s)} \right) \leq q^{(-s)}(\tilde \mu)$. Therefore, $q^{(-s)}\left(\tilde \mu^{(-s)} \right) \leq q^{(-s)}(\tilde \mu)$, which contradicts the fact that $\tilde \mu^{(-s)}$ is the smallest minimizer of $q^{(-s)}(\cdot)$.

        \item Case II: $0 \notin \argmax_{x \in \X} \tilde f_s(x) - \tilde \mu \cdot \tilde b_s(x)$. Since $f_s(x) \leq \kappa \cdot b_s(x)$ for all $x \in \X$, we get that $\tilde b_s(x') > 0$ for all $x' \in \argmax_{x \in \X} \tilde f_s(x) - \tilde \mu \cdot \tilde b_s(x)$. Consider any sequences of optimal action sequences $\{x_t\}_t$ and $\{x_t^{(-s)}\}_t$ such that for all $t \in [T]$, we have
    	\begin{align*}
    		x_t \in \argmax_{x \in \X} \tilde f_t(x) - \tilde\mu \cdot \tilde b_t(x) \quad \text{and} \quad x_t^{(-s)} \in \argmax_{x \in \X} \tilde f_t(x) - \tilde\mu^{(-s)} \cdot \tilde b_t(x) \ .
    	\end{align*}
    	Then, Lemma~\ref{lemma:monotonicity} implies that $\tilde b_t(x_t) \geq \tilde b_t(x_t^{(-s)})$ for all $t \neq s$. Therefore,
    	\begin{align}\label{eq:change-target-inter-1}
    		B - \sum_{t=1}^T \tilde b_t(x_t) = \left\{ B - \sum_{t\neq s} \tilde b_t(x_t) \right\} - b_t(x_t) < B - \sum_{t\neq s} \tilde b_t(x_t) \leq B - \sum_{t\neq s} \tilde b_t(x_t^{(-s)})  \,.
    	\end{align}
    	Now observe that, since $q(\cdot)$ (and $q^{(-s)}(\cdot)$) are the maxima of a collection of linear functions, its sub-gradient is given by the convex hull of gradients of all the linear functions which are binding (for example, see Chapter 5 of \citealt{bertsekas2009convex}). Therefore, $\partial q(\tilde\mu)$ (and $\partial q^{(-s)}\left(\tilde \mu^{(-s)} \right)$) is a convex hull of terms of the form $B - \sum_{t=1}^T \tilde b_t(x_t)$ for some optimal action sequence $\{x_t\}_t$ (and $\{x_t^{(-s)}\}_t$). Since $\tilde \mu^{(-s)} > \tilde \mu \geq 0$, first-order optimality conditions imply that $0 \in \partial q^{(-s)}\left(\tilde \mu^{(-s)} \right)$. Therefore, \eqref{eq:change-target-inter-1} implies that $v < 0$ for all $v \in \partial q(\tilde \mu)$. This contradicts the optimality of $\tilde \mu$ for $q(\cdot)$.
    \end{itemize}
    As we have obtained a contradiction in both cases, we get that $\tilde \mu \geq \tilde \mu^{(-s)}$, as required. Moreover, $\lambda_t \leq \lambda_t^{(-s)}$ for all $t \neq s$ follows immediately from Lemma~\ref{lemma:monotonicity}. Hence, to finish the proof, it suffices to show the final inequality in the following chain:
    \begin{align}\label{eq:change-target-inter-2}
		\sum_{t =1}^{s-1} \left|\lambda_t^{(-s)} - \lambda_t \right| \leq \sum_{t\neq s} \left|\lambda_t^{(-s)} - \lambda_t \right| = \sum_{t\neq s} \lambda_t^{(-s)} - \sum_{t \neq s}\lambda_t \leq 3 \bar b \,. 
	\end{align}
    Note that, Lemma~\ref{lemma:trace-dual} implies that at least one of the following conditions hold
    \begin{enumerate}
		\item $\tilde \mu = 0$ and $\sum_{t=1}^T \lambda_t \leq B + \bar b$.
		\item $\left|B - \sum_{t=1}^T \tilde \lambda_t \right| \leq \bar b$.
	\end{enumerate}
    If $\tilde \mu = 0$, then $\tilde \mu^{(-s)} = 0$ because $\tilde mu^{(-s)} \leq \tilde \mu$. Therefore, in that case $\lambda^{(-s)}_t = \lambda_t = \tilde b^*_t(0)$ for all $t \neq s$ and \eqref{eq:change-target-inter-2} follows. 
    
    Suppose $\left|B - \sum_{t=1}^T \tilde \lambda_t \right| \leq \bar b$. Observe that Lemma~\ref{lemma:trace-dual} applied to the trace $\{\hat \gamma_t\}_t$, where $\hat \gamma_t = \tilde \gamma_t$ for all $t \neq s$ and $\hat \gamma_s = (0,0)$, implies that at least one of the following conditions hold:
    \begin{enumerate}
		\item $\tilde \mu^{(-s)} = 0$ and $\sum_{t\neq s} \lambda^{(-s)}_t \leq B + \bar b$.
		\item $\left|B - \sum_{t\neq s} \tilde \lambda_t^{(-s)} \right| \leq \bar b$.
	\end{enumerate}
    Therefore, $\sum_{t\neq s} \lambda_t^{(-s)} - \sum_{t \neq s}\lambda_t \leq B + \bar b - \sum_{t \neq s}\lambda_t \leq \bar b + \bar b + \lambda_s \leq 3 \bar b$, as required to establish \eqref{eq:change-target-inter-2}.
\end{proof}

\subsection{Proof of Lemma~\ref{lemma:coupling}}

\begin{proof}[Proof of Lemma~\ref{lemma:coupling}]

    It is known that FTRL is equivalent to Lazy Online Mirror Descent (for example, see \citealt{hazan2016introduction}). In particular, if we let $V_h(x,y)=h(x)-h(y)-\nabla h(y)^\top (x-y)$ denote the Bregman divergence w.r.t. $h(\cdot)$, then the FTRL update \eqref{eq:FTRL} of Algorithm~\ref{alg:dual-descent} can be equivalently written as:
	\begin{align*}
		&\theta_s =  \nabla h(\mu_s)\\
		&\theta_{s+1} = \theta_t - \eta \cdot g_s = \theta_t - \eta \cdot (\lambda_s - b_s(x_t'))\\
		&\mu_{t+1} = \argmin_{\mu \in [0,\kappa]} V_h(\mu, (\nabla h)^{-1}(\theta_{s+1}))\,.
	\end{align*}
	where $x_t' \in \argmax_{x\in \X} f_s(x) - \mu_s \cdot b_s(x)$. We will use $\{\mu_t'\}_t$ and $\{\theta_t'\}_t$ to represent the dual and mirror iterates of Algorithm~\ref{alg:dual-descent} with target sequence $\{\lambda'_t\}_t$:
	\begin{align*}
		&\theta'_s =  \nabla h(\mu'_s)\\
		&\theta'_{s+1} = \theta'_t - \eta \cdot g_s = \theta'_t - \eta \cdot (\lambda'_s - b_s(y_t'))\\
		&\mu'_{t+1} = \argmin_{\mu \in [0,\kappa]} V_h(\mu, (\nabla h)^{-1}(\theta'_{s+1}))\,.
	\end{align*}
	where $y_t' \in \argmax_{x\in \X} f_s(x) - \mu'_s \cdot b_s(x)$.

    We will first use induction on $s$ to prove the following statement,
    \begin{align}\label{eq:coupling-induction}
		\left| \theta_s  - \theta'_s \right| \leq \eta \cdot \left\{ \sum_{t=1}^{s-1} |\lambda_t - \lambda_t'| \right\} + \eta \cdot \bar b\,.
	\end{align}
    The base case $s =1$ follows directly from our assumption that the initial iterates $\theta_1 = \nabla h(\mu_1) = \nabla h(\mu_1') = \theta_2$ are the same. 
    
    Suppose \eqref{eq:coupling-induction} holds for $s \in [T-1]$ (Induction Hypothesis). Define $\theta_{s + 1/2} = \theta_s + \eta \cdot b_s(x_t')$ and $\theta'_{s + 1/2} = \theta'_s + \eta \cdot b_s(y_t')$. W.l.o.g. assume that $\theta_s \geq \theta_s'$. Due to the invertibility of $\nabla h$, we get that $\mu_s \geq \mu_s'$, and consequently Lemma~\ref{lemma:monotonicity} implies $b(x_t') \leq b(y_t')$. Consider the following cases:
	\begin{itemize}
		\item Case I: $\theta'_{s+1/2} \leq \theta_{s+1/2}$. Then, $\theta_{s+1/2} - \theta'_{s+1/2} = \theta_s - \theta'_{s} + \eta \cdot (b(x_t') - b(y_t')) \leq \theta_s - \theta_s'$ because $b(x_t') \leq b(y_t')$.
		\item Case II: $\theta'_{s+1/2} \geq \theta_{s+1/2}$. Then, $\theta'_{s+1/2} - \theta_{s+1/2} = \theta'_s - \theta_{s} + \eta \cdot (b(x_t') - b(y_t')) \leq \eta \cdot \bar b$ because $\theta'_{s} \leq \theta_{s}$ and $b(y_t') - b(x_t') \leq \bar b$.
	\end{itemize}
    Therefore, in both cases we have
    \begin{align*}
        |\theta_{s+1/2} - \theta'_{s+1/2}| \leq \max\{\theta_s - \theta'_s, \eta \cdot\bar b\} \leq \eta \cdot \left\{ \sum_{t=1}^{s-1} |\lambda_t - \lambda_t'| \right\} + \eta \cdot \bar b\,.
    \end{align*}
	where we used the induction hypothesis in the second inequality. Consequently, we can write
	\begin{align*}
		|\theta_{s+1} - \theta'_{s+1}| &= \left|\theta_{s+1/2} - \eta \cdot \lambda_s + (\theta'_{s+1/2} - \eta \cdot \lambda'_s) \right|\\
		&\leq |\theta_{s+1/2} - \theta'_{s+1/2}| + \eta \cdot |\lambda_s - \lambda'_s|\\
		&\leq \eta \cdot \left\{ \sum_{t=1}^{s} |\lambda_t - \lambda_t'| \right\} + \eta \cdot \bar b\,.
	\end{align*}
    Hence, we have established \eqref{eq:coupling-induction} for all $s \in [T]$. Now, since $h$ is $\sigma$-strongly convex and differentiable, we have
    \begin{align*}
        \nabla h(x) - \nabla h(y) \geq \sigma \cdot (x - y) \qquad \forall\ x \geq y \,.
    \end{align*}
    Therefore, we have
    \begin{align}\label{eq:coupling-inter-1}
        \left| (\nabla h)^{-1}(\theta_s) - (\nabla h)^{-1}(\theta_s') \right| \leq \frac{1}{\sigma} \cdot |\theta_s - \theta'_s|\,.
    \end{align}

    To finish the proof, we will use the fact that Bregman projections are contractions in one dimensions, which we prove next. Consider any $x < 0$, then for any $\mu \in [0,\kappa]$, we have
    \begin{align*}
        V_h(\mu, x) - V_h(0,x) = h(\mu) - h(0) - \nabla h(x)^\top (\mu - 0) \geq h(\mu) - h(0) - \nabla h(0)^\top (\mu - 0) \geq 0\,,
    \end{align*}
    where the inequality follows from $\nabla h(0) \geq \nabla h(x)$ (convexity of $h(\cdot)$). Therefore, $\argmin_{\mu \in [0,\kappa]} V_h(\mu, x) = 0 = \argmin_{\mu \in [0,\kappa]} |\mu - x|$. Similarly, for $x > \kappa$ and $\mu \in [0,\kappa]$, we have
    \begin{align*}
        V_h(\mu, x) - V_h(\kappa, x) = h(\mu) - h(\kappa) - \nabla h(x)^\top (\mu - \kappa) \geq h(\mu) - h(\kappa) - \nabla h(\kappa)^\top (\mu - \kappa) \geq 0\,,
    \end{align*}
    where the inequality follows from $\nabla h(x) \geq \nabla h(\kappa)$ (convexity of $h(\cdot)$). Therefore, $\argmin_{\mu \in [0,\kappa]} V_h(\mu, x) = \kappa = \argmin_{\mu \in [0,\kappa]} |\mu - x|$. Consequently, we have shown that $\argmin_{\mu \in [0,\kappa]} V_h(\mu, x) = \argmin_{\mu \in [0,\kappa]} |\mu - x|$, i.e., the Bregman project is identical to the Euclidean projection in one dimension. Since Euclidean projection is a contraction, we get
    \begin{align*}
        \left| \mu_s - \mu'_s \right| &= \left| \argmin_{\mu \in [0,\kappa]} V_h(\mu, (\nabla h)^{-1}(\theta_{s+1}))  - \argmin_{\mu \in [0,\kappa]} V_h(\mu, (\nabla h)^{-1}(\theta'_{s+1})) \right|\\
        &= \left| \argmin_{\mu \in [0,\kappa]} |\mu - (\nabla h)^{-1}(\theta_{s+1}))|  - \argmin_{\mu \in [0,\kappa]} |\mu - (\nabla h)^{-1}(\theta'_{s+1}))| \right|\\
        &\leq  \left| (\nabla h)^{-1}(\theta_s) - (\nabla h)^{-1}(\theta_s') \right|\,.
    \end{align*}
    
    Finally, combining this with \eqref{eq:coupling-induction} and \eqref{eq:coupling-inter-1}, we get
    \begin{align*}
        |\mu_s - \mu'_s| \leq  \frac{\eta}{\sigma} \cdot \left\{ \sum_{t=1}^{s-1} |\lambda_t - \lambda_t'| \right\} + \frac{\eta}{\sigma} \cdot \bar b\,,
    \end{align*}
    as required.
\end{proof}

\subsection{Proof of Lemma~\ref{lemma:leave-one-out}}

\begin{proof}[Proof of Lemma~\ref{lemma:leave-one-out}]
    Using the definitions of $\lambda_s$ and $\beta_s$, we can write
    \begin{align*}
        \E\left[ \lambda_s \big| \tilde \mu^{(-s)} \right] = \E\left[ \tilde b_s^*(\tilde \mu) \biggr|\ \tilde \mu^{(-s)} \right] \quad \text{ and } \quad \E\left[ \beta_s \big|\ \tilde \mu^{(-s)} \right] = \E\left[ \E_{\hat \gamma \sim \tilde \PP_t} \left[ \hat b_s^*(\tilde\mu) \right] \biggr|\ \tilde \mu^{(-s)} \right]\,.
    \end{align*}

    Fix a trace $\{\tilde \gamma_t\}_t$. Observe that for any request
    \begin{align*}
        x^*(\tilde \gamma_s, \tilde \mu) = 
        \begin{cases}
            \bar x &\text{if } \coeff(\tilde f_s) - \tilde \mu \cdot \coeff(\tilde b_s) \geq 0 \text{ and } \coeff(\tilde f_s) \neq 0\\
            0 &\text{otherwise} 
        \end{cases}\,,
    \end{align*}
    and
    \begin{align*}
        x^*(\tilde \gamma_s, \tilde \mu^{(-s)}) = 
        \begin{cases}
            \bar x &\text{if } \coeff(\tilde f_s) - \tilde \mu^{(-s)} \cdot \coeff(\tilde b_s) \geq 0 \text{ and } \coeff(\tilde f_s) \neq 0\\
            0 &\text{otherwise} 
        \end{cases}\,.
    \end{align*}
    
    From Lemma~\ref{lemma:change-in-target}, we know that $\tilde \mu^{(-s)} \leq \tilde \mu$. Now, if $\coeff(\tilde f_s) = 0$, then $x^*(\tilde \gamma_s, \tilde \mu) = x^*(\tilde \gamma_s, \tilde \mu^{(-s)}) = 0$. Assume that $\coeff(\tilde f_s) > 0$ (and thus $\coeff(\tilde b_s) > 0$ because $f_s(x) \leq \kappa \cdot b(x)$), let $A \coloneqq \{\mu \geq 0 \mid \coeff(\tilde f_s) - \mu \cdot \coeff(\tilde b_s) < 0\}$ be the set of all dual variables that lead to $x^*(\tilde \gamma_s, \mu) = 0$.

    Define the dual functions:
	\begin{align*}
		q(\mu) \coloneqq \mu\cdot B + \sum_{t=1}^T \max_{x \in \X} \left\{ \tilde f_t(x) - \mu \cdot \tilde b_t(x) \right\} \quad \text{and} \quad q^{(-s)}(\mu) \coloneqq \mu\cdot B + \sum_{t\neq s} \max_{x \in \X} \left\{ \tilde f_t(x) - \mu \cdot \tilde b_t(x) \right\} \,.
	\end{align*}
    
    For contradiction, suppose $\tilde \mu \in A$ and $\tilde \mu^{(-s)} \notin A$. Since $A$ is open, there exists a point $\mu \in A$ such that $\mu = \alpha \cdot \tilde \mu + (1 - \alpha) \cdot \tilde \mu^{(-s)}$ for some $\alpha \in (0,1)$. Moreover, observe that the minimality of $\tilde \mu^{(-s)}$ implies $q^{(-s)}\left(\tilde \mu^{(-s)} \right) \leq q^{(-s)}(\mu)$ and  $q^{(-s)}\left(\tilde \mu^{(-s)} \right) \leq q^{(-s)}(\tilde \mu)$. Therefore, as $q^{(-s)}$ is convex, we get
    \begin{align*}
        q^{(-s)}(\mu) \leq \alpha \cdot q^{(-s)}(\tilde \mu) + (1 - \alpha) \cdot q^{(-s)}\left(\tilde \mu^{(-s)} \right) \leq \alpha \cdot q^{(-s)}(\tilde \mu) + (1 - \alpha) \cdot q^{(-s)}(\tilde \mu^) = q^{(-s)}(\tilde \mu)\,.
    \end{align*}
    Now, observe that $q(\mu) = q^{(-s)}(\mu)$ for all $\mu \in A$. Therefore, $q(\mu) \leq q(\tilde \mu)$, which contradicts the fact that $\tilde \mu$ is the smallest minimizer of $q(\cdot)$. Hence, either $\tilde \mu, \tilde \mu^{(-s)} \in A$ or $\tilde \mu, \tilde \mu^{(-s)} \in A^c$, and as a consequence, we get $\tilde b_s^*(\tilde \mu) = \tilde b_s^*\left(\tilde \mu^{(-s)} \right)$. Furthermore, combining $\tilde \mu \geq \tilde \mu^{(-s)}$ (from Lemma~\ref{lemma:change-in-target}) and Lemma~\ref{lemma:monotonicity}, we also get $\hat b_s^* \left(\tilde \mu^{(-s)} \right) \geq \hat b_s^*(\tilde \mu)$ for every $\hat \gamma_s \in \S$. Therefore,
    \begin{align*}
        \E\left[ \lambda_s \big| \tilde \mu^{(-s)} \right] &= \E\left[ \tilde b_s^*(\tilde \mu) \biggr|\ \tilde \mu^{(-s)} \right]\\
        &= \E\left[ \tilde b_s^*\left(\tilde \mu^{(-s)} \right) \biggr|\ \tilde \mu^{(-s)} \right]\\
        &= \E_{\hat \gamma_s \sim \tilde \PP_s}\left[ \hat b_s^*\left(\tilde \mu^{(-s)} \right) \biggr|\ \tilde \mu^{(-s)} \right]\\
        &\geq \E\left[ \E_{\hat \gamma_s \sim \tilde \PP_s}\left[ \hat b_s^*(\tilde \mu) \biggr|\ \tilde \mu^{(-s)} \right] \right]\\
        &= \E\left[ \beta_s \big|\ \tilde \mu^{(-s)} \right]\,,
    \end{align*}
    where the third equality follows from the fact that $\tilde \gamma_s$ and $\tilde \mu^{(-s)}$ are independent of each other, which allows us to rename the variable from $\tilde \gamma_s \sim  \tilde \PP_s$ to $\hat \gamma_s \sim \tilde \PP_s$. We combine this with the Tower Property of conditional expectations to finish the proof:
    \begin{align*}
        \E \left[ \mu_s^{(-s)} \cdot (\beta_s - \lambda_s) \right] = \E\left[ \mu_s^{(-s)} \cdot \E \left[  (\beta_s - \lambda_s) \biggr|\ \tilde \mu^{(-s)}  \right] \right] = \E\left[ \mu_s^{(-s)} \cdot \left(  \E \left[\beta_s \biggr|\ \tilde \mu^{(-s)}  \right] - \E \left[\lambda_s \biggr|\ \tilde \mu^{(-s)}  \right] \right) \right] \leq 0 \,. 
    \end{align*}
    \qedhere
\end{proof}

\subsection{Proof of Theorem~\ref{thm:main-result}}

\begin{proof}[Proof of Theorem~\ref{thm:main-result}]
    Theorem~\ref{thm:dual-descent-regret} and \eqref{eq:discrepancy-term} together imply that, with probability at least $1 - 1/T^2$, we have

    \begin{align*}
        \regret(A) &= \fluid(\{\PP_t\}_t) - \E_{\{\gamma_t\}_t \sim \prod_t \PP_t}[R(A|\{\gamma_t\}_t)]\\
         &\leq R_1 + R_2 + R_3 + \kappa \cdot r(T) + \kappa \bar b + 2(1 + \kappa) \cdot \sum_{t=1}^T \W(\PP_t, \tilde \PP_t) \,.
    \end{align*}

    From our choice of step size $\eta = \sqrt{d_R/T}$ and the observation that $\bar \lambda = \max_t \lambda_t \leq \bar b$, we get that 
    \begin{align*}
        R_1 = \kappa \bar b + \frac{2(\bar b+ \bar \lambda)^2}{\sigma} \cdot \eta T + \frac{d_R}{\eta} \leq \kappa \bar b + \left( \frac{8 \bar b^2}{\sigma} + 1 \right) \cdot \sqrt{d_r T} \,.
    \end{align*}
    From \eqref{eq:excess-spend}, we know that
    \begin{align*}
        R_2 = \kappa \cdot \left( \left\{ \sum_{t=1}^T \lambda_t \right\} - B \right)^+\leq \kappa\bar b\,.
    \end{align*}
    Moreover, from Lemma~\ref{lemma:leave-one-out} and $\eta = \sqrt{d_R/T}$, we know that
    \begin{align*}
        R_3 = \sum_{s=1}^T \E \left[ \mu_s \cdot (\beta_s - \lambda_s) \right] \leq \frac{4 \eta \bar b^2}{\sigma} \cdot T = \frac{4\bar b^2}{\sigma} \cdot \sqrt{d_r T}\quad\,.
    \end{align*}

    Combining the above inequalities and plugging in $r(T) = 8 \bar b \sqrt{T \log(T)}$, we get that
    \begin{align*}
        \regret(A) \leq 3 \kappa \bar b + \left( \frac{12 \bar b^2}{\sigma} + 1 \right) \cdot \sqrt{d_r T} + 8 \kappa \bar b \sqrt{T \log(T)} + 2(1 + \kappa) \cdot \sum_{t=1}^T \W(\PP_t, \tilde \PP_t)
    \end{align*}
    with probability at least $1 - 1/T^2$. On the other than, we always have $\regret(T) \leq \bar f T \leq \kappa \bar b T$. Hence, we get
    \begin{align*}
        \regret(A) &\leq \left(1 - \frac{1}{T^2} \right) \cdot \left[ 3 \kappa \bar b + \left( \frac{12 \bar b^2}{\sigma} + 1 \right) \cdot \sqrt{d_r T} + 8 \kappa \bar b \sqrt{T \log(T)} + 2(1 + \kappa) \cdot \sum_{t=1}^T \W(\PP_t, \tilde \PP_t) \right] + \frac{1}{T^2} \cdot \kappa \bar b T\\
        &\leq 4 \kappa \bar b \sqrt{T \log(T)} + \left( \frac{12 \bar b^2}{\sigma} + 1 \right) \cdot \sqrt{d_r} \cdot \sqrt{T \log(T)} + 8 \kappa \bar b \sqrt{T \log(T)} + 2(1 + \kappa) \cdot \sum_{t=1}^T \W(\PP_t, \tilde \PP_t)\\
        &= C_1 \cdot \sqrt{T \log(T)} + C_2 \cdot \sum_{t=1}^T \W(\PP_t, \tilde \PP_t)
    \end{align*}
    where $C_1 = \frac{12 \bar b^2 \sqrt{d_R}}{\sigma} +  \sqrt{d_R} + 12 \kappa \bar b$ and $C_2 = 2(1+\kappa)$.
\end{proof}

\section{Numerical Experiments}

We evaluated the algorithms presented in this paper on ad auction data from a major online video platform. This section presents the results. We start with an overview of the experimental setup, including a description of the implementation details of the algorithms.

\textbf{Experimental Setup.} We considered the problem of value (click) maximization subject to budget constraints. To smooth out noise and allow for periodic patterns, we divided each day into 15 min intervals and grouped together auctions within each interval. This resulted in $T = (24 \times 60 )/15 = 96$ total intervals, and we treated each such interval as a single time period. In other words, all algorithms used the same dual multiplier $\mu$ for all auctions within each interval. This can be seen as a batched version of our model where each time period only had one auction. We compared the performance of 1003 campaigns on data from a Wednesday in September, with the samples being generated from the Tuesday of the same week. We evaluated the following three algorithms:
\begin{itemize}
    \item Static Dual Policy (Algorithm~\ref{alg:learned-dual}): Learn the optimal dual multiplier on the samples from Tuesday and used it as the static multiplier for the test data from Wednesday till the budget ran out.

    \item Dual FTRL with Constant Target: Dual FTRL (Algorithm~\ref{alg:dual-descent} with constant target consumption sequence $\lambda_t = B/T$ for all $t \in [T]$. This algorithm is known to be optimal in the adversarial setting~\citep{balseiro2022best}. We refer to this algorithm as BLM since it was originally proposed and analyzed in \citep{balseiro2022best}.

    \item Dual FTRL with Sample-Optimal Target:  Dual FTRL (Algorithm~\ref{alg:dual-descent}) with the optimal spend rate learned from the samples on Tuesday. Algorithm~\ref{alg:efficient-trace-dual} was used to compute the optimal dual solution $\tilde \mu$ on the samples, and the target consumption for a time period was set equal to the total spend under $\tilde \mu$ in all auctions in that period. We refer to this algorithm as FTRL.
\end{itemize}
We used the same step size $\eta = 0.5 \times \rho/\sqrt{T}$ for both FTRL and BLM. This value was chosen to ensure stability of iterates, and both FTRL and BLM yielded their best performance for this value among all step sizes we tried. We measured the performance of all algorithms using their competitive ratio, which is defined as the ratio of the total value obtained under the algorithm to the total value obtained under the hindsight-optimal solution.

\begin{figure}
    \centering
    \includegraphics[width=0.8\textwidth]{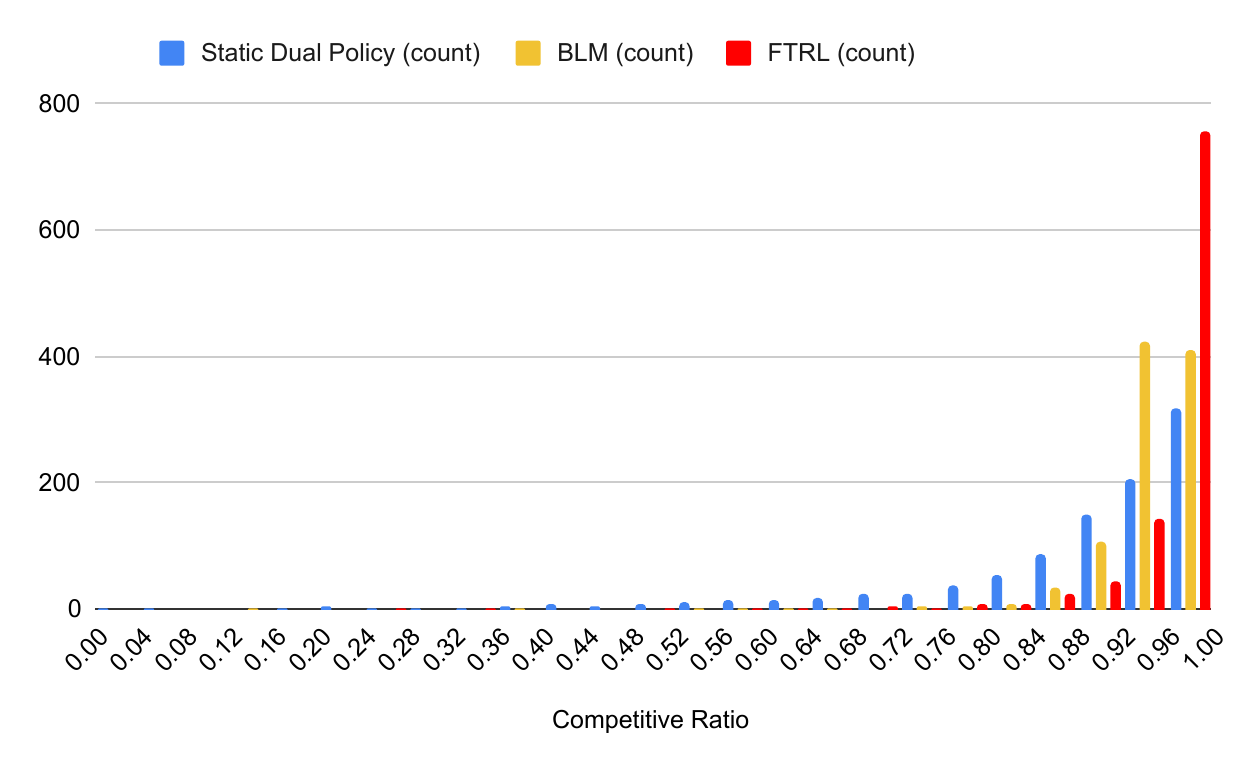}
    \caption{The histogram of the competitive ratios attained by the Static Dual Policy, FTRL and BLM on 1,003 ad campaigns for a major video platform.}
    \label{fig:yt}
\end{figure}

\begin{table}[]
\centering
\begin{tabular}{rllll}
\multicolumn{1}{l}{}                    & Mean                       & Median                     & Std. Error                 &  \\ \cline{2-4}
\multicolumn{1}{r|}{Static Dual Policy} & \multicolumn{1}{l|}{0.874} & \multicolumn{1}{l|}{0.925} & \multicolumn{1}{l|}{0.005} &  \\ \cline{2-4}
\multicolumn{1}{r|}{BLM}                & \multicolumn{1}{l|}{0.943} & \multicolumn{1}{l|}{0.956} & \multicolumn{1}{l|}{0.002} &  \\ \cline{2-4}
\multicolumn{1}{r|}{FTRL}               & \multicolumn{1}{l|}{0.962} & \multicolumn{1}{l|}{0.980} & \multicolumn{1}{l|}{0.002} &  \\ \cline{2-4}
\end{tabular}
\caption{Mean, median and standard error of the competitive ratios of Static Dual Policy, BLM and FTRL.}
\label{table:summary_stat}
\end{table}
\textbf{Results.} Figure~\ref{fig:yt} and Table~\ref{table:summary_stat} summarize the results of our experiments. The results show that FTRL performs much better than the static dual policy, highlighting the fragility of the latter to distribution shifts. In particular, we found that 148 campaigns achieved a competitive ratio of less than 0.75 under the Static Dual Policy, whereas that number was only 17 for FTRL. Somewhat surprisingly, we also found that BLM does nearly as well as FTRL on real data. Overall, FTRL had the best mean and median competitive ratio. The results indicate the power of dynamic policies, as witnessed by the superior performance of FTRL and BLM over the Static Dual Policy, and the need for a nuanced approach to computing the target spend curve---a constant spend curve may outperform the optimal curve obtained from historical samples if the past is not indicative of the future.

\end{document}